\documentclass{article}
\usepackage{ijcai20}

\usepackage{times}
\usepackage{soul}
\usepackage{url}
\usepackage[utf8]{inputenc}
\usepackage[small]{caption}
\usepackage{graphicx}
\usepackage{amsmath}
\usepackage{amsthm}
\usepackage{booktabs}
\usepackage{algorithm}
\usepackage{algorithmic}
\usepackage{amssymb}
\usepackage{enumitem}
\usepackage{multirow}
\usepackage[marginal]{footmisc}
\title{Deep Robust Multilevel Semantic Cross-Modal Hashing}

\author{
Ge Song$^{1,2,3}$
\and
Jun Zhao$^{4}$\and
Xiaoyang Tan$^{1,2,3}$\thanks{Corresponding author}
\affiliations
$^1$College of Computer Science and Technology, Nanjing University of Aeronautics and Astronautics \\
$^2$MIIT Key Laboratory of Pattern Analysis and Machine Intelligence, China \\
$^3$Collaborative Innovation Center of Novel Software Technology and Industrialization, China\\
$^4$Nanyang Technological University, Singapore\\
\emails
\{sunge, x.tan\}@nuaa.edu.cn, junzhao@ntu.edu.sg
}

\begin{document}

\maketitle

\begin{abstract}
Hashing based cross-modal retrieval has recently made significant
progress. But straightforward embedding data from different
modalities involving rich semantics into a joint Hamming space will
inevitably produce false codes due to the intrinsic modality
discrepancy and noises. We present a novel Robust Multilevel
Semantic Hashing (RMSH) for more accurate multi-label cross-modal
retrieval. It seeks to preserve fine-grained similarity among data
with rich semantics,i.e., multi-label, while explicitly require
distances between dissimilar points to be larger than a specific
value for strong robustness. For this, we give an effective bound of
this value based on the information coding-theoretic analysis, and
the above goals are embodied into a margin-adaptive triplet loss.
Furthermore, we introduce pseudo-codes via fusing multiple hash
codes to explore seldom-seen semantics, alleviating the sparsity
problem of similarity information. Experiments on three benchmarks
show the validity of the derived bounds, and our method achieves
state-of-the-art performance.\end{abstract}

\section{Introduction}
\label{sec:intro} Cross-modal retrieval, aiming to search similar
instances in one modality with the query from another, has gained
increasing attention due to its fundamental role in large-scale
multimedia
applications~\cite{cao2019video,wang2019camp,dutta2019semantically,zhu2019r2gan}.
The difficulty of the similarity measurement of data from different
modals makes this task very challenging, which is known as the
heterogeneity gap~\cite{BaltrusaitisAM19}. An essential idea to
bridge this gap is mapping different modals into a joint feature
space such that they become computationally comparable, and it is
widely exploited by previous
work~\cite{XuLYDL19,ZhenHWP19,Zhan0YZT018}. Among them,
hashing-based method~\cite{LiuLZWHJ18,ShiYZWP19}, embedding the data
of interest into a low-dimensional Hamming space, has gradually
become the mainstream approach due to the low memory usage and high
query speed of binary codes.

\begin{figure}
\begin{tabular}{c}
\includegraphics[scale=.4]{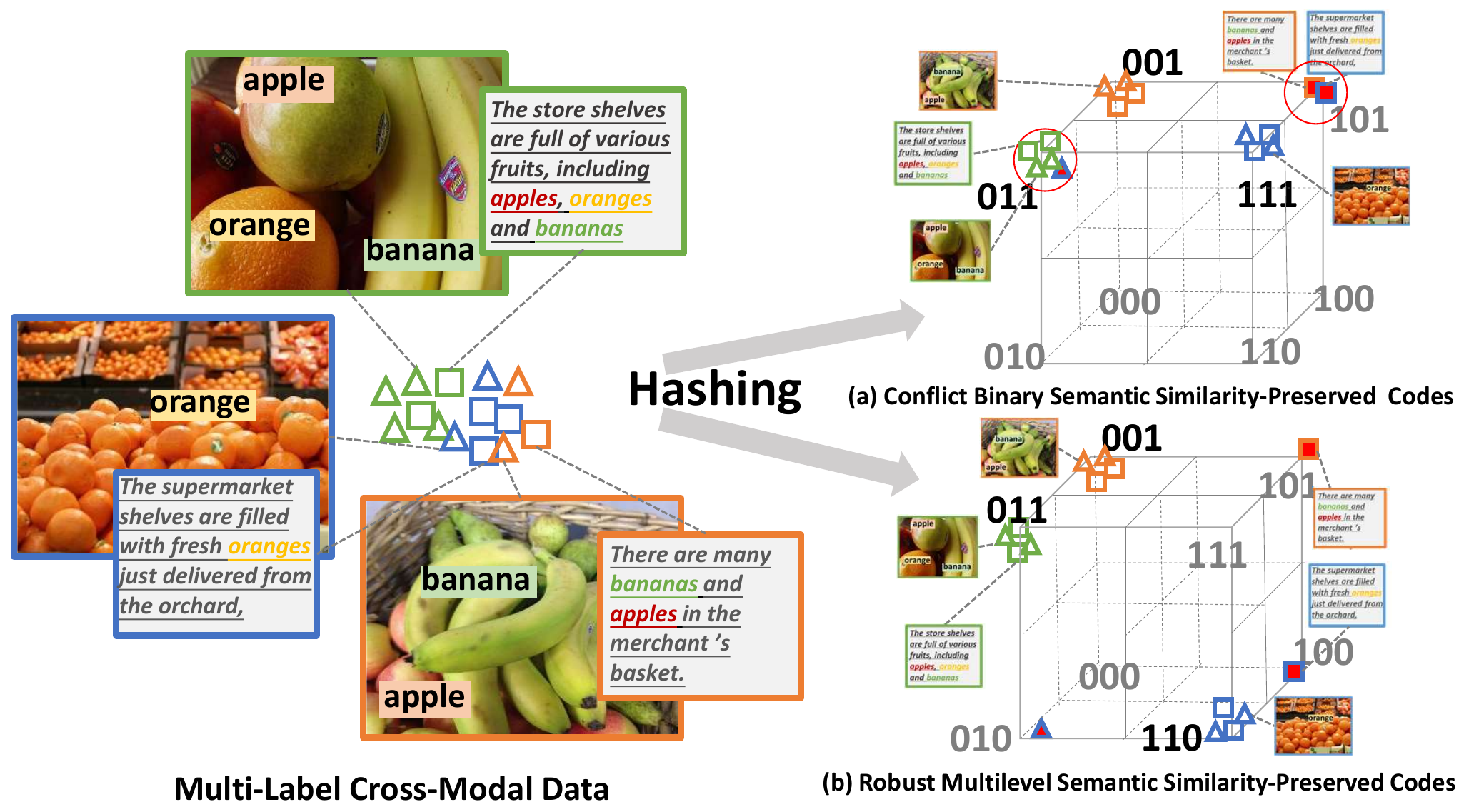}
\end{tabular}
\caption{Illustration of the robust multilevel semantic
similarity-preserved codes for multi-label cross-modal data. The
points filled with red color in Hamming space denote the false codes
caused by the intrinsic noise of data.}
\label{fig:coding_efficiency}
\end{figure}

While recent works have made significant progress, there are several
drawbacks still in existing cross-modal hashing methods. First,
embedding different modalities data sharing the same semantic into
the unified hash codes is hard, since the inherent modality
discrepancy and noises will inevitably cause false codes.
Nevertheless, most approaches learn to hash straightforwardly with
seldom considerations for this problem. They tend to embed data of
different semantics into adjacency vertexes in Hamming space, which
dramatically increases the collision probability of correct and
false hash codes. Although some works~\cite{Jiang2017Deep} mentioned
this problem and introduced the bit balance~\cite{WangZSSS18}
constraint for maximizing the information provided by each bit, it
is too simple and leads to the burden of seeking proper
hyper-parameter for effective learning. Second, the query always is
complicated in real applications, involving rich semantics, e.g.,
multi-labels. But numerous work could not run such queries to return
satisfying results that are consistent with the humans' cognition on
semantic similarity. They~\cite{Jiang2017Deep,HuWZP19} focus on
preserving simple similarity structures (i.e., similar or
dissimilar) rather than more fine-grained ones, and the used
similarity information is often very sparse. We illustrate these
problems in Fig.~\ref{fig:coding_efficiency} (a), the image and text
data associate with multiple tags, and their similarity score is not
just 0 or 1. If learning hash for only keeping 0-1 similarity
without proper constraints, i.e., push away the dissimilar points
and pull the similar points, the conflict codes will happen,
e.g.,'101'. Besides, the data of 'apple, banana' is more similar
than that of 'orange' to the data of 'apple, orange, banana,' but
the two former's hashing codes have the same Hamming distance with
the later's.

We observe that if the representation ability of binary codes with
fixed length is adequate, hashing functions should attempt to
preserve the complete fine-grained similarity structure for more
accurate retrieval. Simultaneously, we can explicitly impose the
distances between codes, whose similarities are zero, to be larger
or equal than a specific value $\delta$ to make the learned hash
codes more robust. We call $\delta$ as {\it robust parameter} and
the learned codes as {\it robust multilevel semantic
similarity-preserved codes}. As shown in
Fig.~\ref{fig:coding_efficiency}(b), we hash multi-label cross-modal
data into a 3-bit Hamming space according to their subtle similarity
so that the distance of irrelevant data's codes is larger or equal
than 3 (i.e., the 'orange' and 'apple, banana' data) and the
distance of the more relevant data's codes (i.e., the 'apple,
banana' and 'orange, apple, banana' data) is smaller than that of
normal relevant data's codes (i.e., the 'orange' and 'orange, apple,
banana' data), then no conflict happens. Intuitively, the larger
$\delta$ is, the more robust learned codes are, and the more levels
of the semantic similarity can be encoded. We could embody this
constraint in the objective of hashing learning. However, endowing
too large $\delta$ is not practical due to the limited coding power
of length-fixed binary codes and the uncertainty of the hashing
process. The question thus becomes how to find an appropriate
$\delta$. In theory, as the Hamming space and the semantic
similarity information are definite, the assumption that all data
are well hashing, i.e., no false codes, can be helpful to reduce
uncertainty and ease the derivation of the range of effective
$\delta$.

Here, we briefly describe our answer to the above question under the
previous assumption. We would like to encode semantic and similarity
information of data to $K$-bit binary codes. The maximum number of
$K$-bit codes with minimum pairwise Hamming distance $\delta$ is
certain. According to the information coding theory, the log of this
number should be larger than the amount of semantic information, and
the $\delta$-bits should be able to encode the neighbor similarity
information of each point. Based on these facts, we derive the
bounds of proper $\delta$ and detail the process in
Sec.~\ref{sec:emm}.

Inspired by the above observation, we propose a novel deep Robust
Multilevel Semantic Hashing (RMSH), which treats preserving the
complete multilevel semantic similarity structure of multi-label
cross-modal data with theoretically guaranteed distance constraint
between dissimilar data, as the objective to learn hash functions.
For this, a margin-adaptive triplet loss is adopted to control the
distance of dissimilar points in Hamming space explicitly, meanwhile
embedding similar points with a fine-grained level. To alleviate the
sparsity problem of similarity information, we further present
fusing multiple hash codes at the semantic level to generate
pseudo-codes, exploring the seldom-seen semantics. The main
contributions are summarized as follows.
\begin{itemize}[leftmargin=1em,topsep=0em]
\item A novel hashing method, named RMSH, is proposed to learn multilevel semantic similarity-preserved codes for accurate multi-label cross-modal retrieval. In which, to exploit the finite Hamming space for improving the robustness of learned codes, we require the distance between codes of dissimilar points satisfies larger than a specific value. For more effective learning, we perform a theoretical analysis of this value and bring out its effective range.
\item To capture the fine-grained semantic similarity structure in coupling with the elaborated distance constraint, we present a margin-adaptive triplet loss. Moreover, a new pseudo-codes network is introduced, tailored to explore more rare and complicated similarity structures.
\item Extensive experimental results demonstrate the effectiveness of the derived bounds, and the proposed RMSH approach yields the state-of-the-art retrieval performance on three cross-modality datasets.
\end{itemize}

\section{Related work}
\label{sec:relatedwork} In this section, we briefly review related
works concerning cross-modal hashing methods.

The cross-modal
hashing~\cite{cao2018cross,zhang2018attention,chen2019two,lu2019flexible,li2019coupled,su2019deep}
can be grouped into two types, unsupervised and supervised. The
former utilizes the co-occurrence information of the multi-modal
pair (e.g., image-text) to maximize their correlation in the common
Hamming space. The representative is Collective Matrix Factorization
Hashing (CMFH)~\cite{Ding2014Collective}, which generates unified
hash codes for multiple modalities by performing collective matrix
factorization from different views. The supervised ones aim to
preserve semantic similarity. Semantic Correlation Maximization
(SCM)~\cite{Zhang2014Large} uses hash codes to reconstruct semantic
similarity matrix. Semantics-Preserving Hashing
(SePH)~\cite{Lin2015Semantics} minimizes KL-divergence between the
hash codes and semantics distributions. Recently, the success of
deep learning prompted the development of cross-modal hashing.
Cross-Modal Deep Variational Hashing (CMDVH)~\cite{LiongLT017}
infers fusion binary codes from multi-modal data as the latent
variable for model-specific networks to approximate. Self-supervised
Adversarial Hashing (SSAH)~\cite{LiDL0GT18} incorporates the
adversarial learning into cross-modal hashing. Equally-Guided
Discriminative Hashing (EGDH)~\cite{ShiYZWP19} jointly considers
semantic structure and discriminability to learn hash functions.
Similar to CMDVH, but Separated Variational Hashing Networks
(SVHNs)~\cite{HuWZP19} learns binary codes of semantic labels as the
latent variable. Despite their effectiveness, most of them ignore
the exploitation of limited representative of binary codes for
reducing the impact of the modality gap, and they fail to preserve
fine similarity for more accurate multi-label cross-modal retrieval.

\begin{figure*}[t]
\begin{center}
\includegraphics[width=1\linewidth]{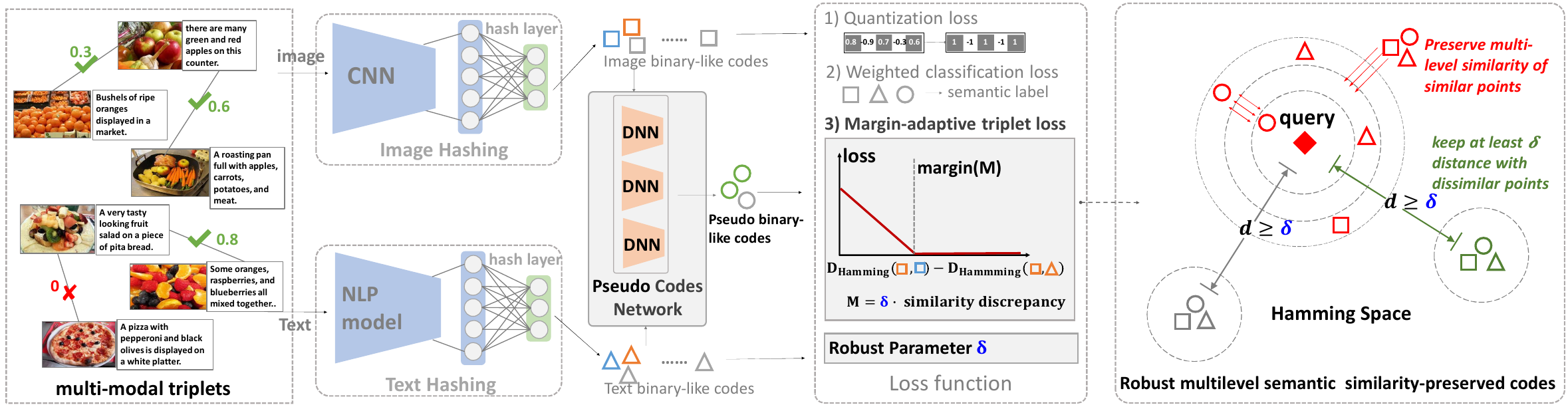}
\end{center}
    \caption{The architecture of the proposed RMSH includes two modality-specific hashing networks and a pseudo-codes network (PCN). The RMSH first takes multi-modal triplets into hashing networks to obtain hash codes, which are then manipulated by PCN to generated codes with seldom-seen semantics. Finally, a margin-adaptive triplet loss equipped with the pre-computed robust parameter $\delta$ is coupled with the weighted classification loss and quantization loss to learn robust multilevel semantic similarity-preserved codes.}
\label{fig:method}
\end{figure*}
\section{The proposed approach}
\label{sec:method} In this section, we first give the problem
definition and the effective bounds of the robust parameter defined
in Sec.~\ref{sec:intro}. Then, the configurations of the proposed
RMSH are detailed, including two hashing networks, a pseudo-code
network, and the corresponding objective function. An overview of
the RMSH is illustrated in Fig.~\ref{fig:method}.

\subsection{The Problem Definition}
\label{sec:problem} Without loss of generality, we use bi-modality
(e.g., image and text) retrieval for illustration in this paper.
Assume that we are given a multi-modal dataset
$\mathbf{D}=\{\mathbf{X}$, $\mathbf{Y}$, $\mathbf{L}\}$, where
$\mathbf{X}=\{x_i\}_{i=1}^{N}$, $\mathbf{Y}=\{y_i\}_{i=1}^{N}$,
$\mathbf{L}=\{l_i\}_{i=1}^{N}$ denotes the image modality, text
modality, and corresponding semantic labels, respectively. $x_i$ can
be features or raw pixels of images, $y_i$ is textual description or
tags, semantic label $l_i\in \{0,1\}^C$, where C is the number of
class. For any two samples, we define the cross-modal similarity
$\mathbf{S}$ as $S_{ij}=\frac{|l_i\cap l_j|}{max\{|l_i|,| l_j|\}}$.
Given training data $\mathbf{X}$, $\mathbf{Y}$ and $\mathbf{S}$, our
goal is to learn two hash functions:
$h^{(x)}(x)=b^{(x)}\in\{-1,1\}^K$ for the image modality and
$h^{(y)}(y)=b^{(y)}\in \{-1,1\}^K$ for the text modality, where $K$
is the length of the binary code, such that the similarity
relationship $\mathbf{S}$ is preserved and distances between
dissimilar data are larger or equal than a positive integer $\delta$
$(<K)$ for robustness, i.e., $\forall$ $x_i\in \mathbf{X}$, $y_j$,
$y_k\in \mathbf{Y}$, if $S_{ij}\ge S_{ik}$, then the Hamming
distance of their binary codes should satisfy
$d_{H}(b_i^{(x)},b_j^{(y)})$ $\le$ $d_{H}(b_i^{(x)},b_k^{(y)})$, and
vice versa, if $S_{ij}=0$, then $d_{H}(b_i^{(x)},b_j^{(y)})$ $\ge$
$\delta$, $\delta$ is named as {\it robust parameter}. Intuitively,
a larger $\delta$ makes codes more robust, but it cannot be too
large due to the finite representation power of $K$-bit binary
codes. In what follows, we investigate how to properly set this
parameter to obtain the balance between robustness and compactness.
\subsection{Bounds of Effective Robust Parameter}
\label{sec:emm} To learn robust hashing codes effectively, we give
the range of effective robust parameter $\delta$ in this section. In
essential, the goal of deep supervised hashing is encoding semantic
information $H(\pmb{L})$ and similarity information
$H(\pmb{S_{*,:}})$ of data\footnote{The semantic label $l$ of data
can be seen as the i.i.d. random variable from distribution $P(L)$.
Because $S$ is constructed by $l$, each row $S_{*,:}$ of $S$ can
also be seen as random variable.} by $K$-bit binary codes, where $H$
denotes entropy function. Therefore, according to the coding
theory~\cite{guruswami2012essential}, we have the following two
facts. {\bf (1)} To make sure that different semantics have unique
codes and the distances between the codes of irrelevant samples are
larger or equal than $\delta$, the number of $K$-bit binary codes
that satisfy pairwise minimum distance $\delta$ should be larger
than $2^{H(\pmb{L})}$. {\bf (2)} $\delta$ bits should be able to
encode the neighborhood semantic similarity information
$H(\pmb{S_{*,:}})$ of each sample $*$. Based on these two facts, we
derive the bounds of effective $\delta$ as follows. For simplicity's
sake, we assume that all data sharing the same semantic are embedded
into the same codes to eliminate the influence of noise.

\noindent {\bf 1) Upper bound.} By the fact (1), we have:
 \begin{equation}
 \label{eq:fact1_formula}
\begin{aligned}
 2^{H(\pmb{L})}\le A(K, \delta)
 \end{aligned}
\end{equation}
 where $A(K, \delta)$ denotes the maximum number of K-bit binary codes with pairwise minimum Hamming distance $\delta$, i.e., $A(K, \delta)=|\{b|b\in\{-1,1\}^K, for\ \forall\ b_i\ne b_j, d_H(b_i,b_j)\ge\delta\}|$. To obtain the range of effective $\delta$, we need estimate $A(K, \delta)$. However, an accurate estimation of $A(K, \delta)$ is challenging, which is still unsolved. Alternatively, its bounds are well studied~\cite{HelgertS73,AgrellVZ01,BelliniGS14}. Towards our goal, we introduce a well-known lower-bound of $A(K, \delta)$ to derive the upper-bound of $\delta$.
\newtheorem{myDef}{Definition}
\newtheorem{myproof}{proof}
\newtheorem{myTheo}{Theorem}
\newtheorem{myLemma}{Lemma}

\begin{myLemma}\label{lem:count_of_balls}(Gilbert-Varshamov Bound~\cite{gilbert1952comparison}).
Given $\delta$, $K\in \mathbb{N}$ be such that $\delta\le K$, we
have:
 \begin{equation}
 \label{eq:count_of_balls}
\begin{aligned}
\frac{2^K}{\sum_{i=0}^{\delta-1}\binom{K}{i}}\le A(K,\delta)
 \end{aligned}
\end{equation}
\end{myLemma}

To proof the Lemma~\ref{lem:count_of_balls}, we first give the
definition of Hamming Ball and then give the proof.

\begin{myDef}\label{def:hammming_ball}(Hamming Ball). Let $\delta$, $K\in \mathbb{N}$ such that $\delta\le K$. For $\forall$ $x$ $\in \Omega=\{-1,1\}^K$, the ball $B_{\delta}^K(x)$ denotes the set of vectors with distance from the $x$ less than or equal to $\delta$ and is defined as $B_{\delta}^K(x)=\{y\in\Omega| d_H(y,x)\le \delta\}$. Its volume is defined as $|B_{\delta}^K(x)|=\sum_{i=0}^{\delta}\binom{K}{i}$, and $|B_{\delta}^K|$ for short.
\end{myDef}

\begin{proof}\renewcommand{\qedsymbol}{}\label{prf:count_of_balls}
Let $\Omega^{(\delta)}$ denotes the greatest subset of
$\Omega=\{-1,1\}^K$ that $\forall$ $c, c'$ $\in$
$\Omega^{(\delta)}$, $d_H(c,c')\ge$ $\delta$ and
$A(K,\delta)=|\Omega^{(\delta)}|$, $\varepsilon=\delta-1$. For
$\forall$ $y\in$ $\Omega$, there exists at least one $c\in$
$\Omega^{(\delta)}$ such that $d_H(y,c)\le \varepsilon$, otherwise
y$\in$ $\Omega^{(\delta)}$. Hence, $\Omega$ is contained in the
union of all balls $B_{\varepsilon}^K(c)$, $c\in$
$\Omega^{(\delta)}$, that is $\Omega \subseteq \bigcup_{c \in
\Omega^{(\delta)}}B_{\varepsilon}^K(c)$, then we have $|\Omega|\le
\sum_{c \in
\Omega^{(\delta)}}|B_{\varepsilon}^K(c)|=|\Omega^{(\delta)}|\cdot|B_{\varepsilon}^K(c)|$,
thus
$A(K,\delta)=|\Omega^{(\delta)}|\ge\frac{|\Omega|}{|B_{\varepsilon}^K(c)|}=
\frac{2^K}{\sum_{i=0}^{\delta-1}\binom{K}{i}}$.
\end{proof}

 As we see that if $2^{H(\pmb{L})}$ is smaller than the lower-bound of $A(K,\delta)$, Eq.~(\ref{eq:fact1_formula}) also holds. By the Lemma~\ref{lem:count_of_balls}, we thus have:
 \begin{equation}
 \label{eq:fact1_further_formula}
\begin{aligned}
 2^{H(\pmb{L})}\le \frac{2^K}{\sum_{i=0}^{\delta-1}\binom{K}{i}}
 \end{aligned}
\end{equation}

Still, the computation of $\sum_{i=0}^{\delta-1}\binom{K}{i}$ in
Eq.~(\ref{eq:fact1_further_formula}) is hard. Here, we show an
upper-bound of this value via the entropy bound of the volume of a
Hamming ball.

\begin{myLemma}\label{lem:volume_of_Hamming_ball}(Volume of Hamming ball).
Given $\delta$, $K\in \mathbb{N}$ and $\delta=pK$, $p\in$ $[0,1/2]$,
we have:
 \begin{equation}
 \label{eq:volume_of_Hamming_ball}
\begin{aligned}
|B^K_{\delta}|\le 2^{H(p)K}
 \end{aligned}
\end{equation}
where $H(p)$ $=-p\log p-(1-p)\log (1-p)$.
\end{myLemma}

\begin{proof}\renewcommand{\qedsymbol}{}\label{prf:volume_of_Hamming_ball}
\begin{equation}\nonumber
\begin{aligned}
\frac{|B^K_{\delta}|}{2^{H(p)K}}&=\frac{\sum_{i=0}^{\delta}\binom{K}{i}}{p^{-pK}(1-p)^{-(1-p)K}}\\
&=\sum_{i=0}^{pK}\binom{K}{i}(1-p)^K(\frac{p}{1-p})^{pK}\\
&<\sum_{i=0}^{pK}\binom{K}{i}(1-p)^K(\frac{p}{1-p})^{i} (\mathrm{since}\ \frac{p}{1-p}<1, i\le pK)\\
&=\sum_{i=0}^{pK}\binom{K}{i}(1-p)^{K-i}p^{i}=1(\mathrm{binomial}\ \mathrm{theorem})\\
\end{aligned}
\end{equation}
\end{proof}

By the Lemma~\ref{lem:volume_of_Hamming_ball} and
Eq.~(\ref{eq:fact1_further_formula}), we can obtain a computable
lower-bound of $A(K,\delta)$. By this bound, we give the theorem
about the upper bound of effective $\delta$ in our problem.

\begin{myTheo}\label{thm:dissimilar_hamming_distance_satisfitation}
For information $H(\pmb{L})$, if using $K$-bit binary codes to
encode $l\in\pmb{L}$ such that for $\forall$ $l_i\ne l_j$,
$d_{H}(b_i,b_j)\geq \delta$, $\delta\le K/2$, then, $\delta$ should
satisfy:
\begin{equation}
\label{eq:upper_bound}
\begin{aligned}
H\big(\frac{\delta-1}{K}\big) \leq 1-\frac{H(\mathbf{\pmb{L}})}{K}
 \end{aligned}
\end{equation}
\end{myTheo}

\begin{proof}\renewcommand{\qedsymbol}{}\label{prf:dissimilar_hamming_distance_satisfitation}
By the the Lemma~\ref{lem:volume_of_Hamming_ball} and to make
Eq.~(\ref{eq:fact1_further_formula}) always hold, we have:
\begin{equation}\nonumber
\begin{aligned}
2^{H(\mathbf{\pmb{L}})}\le \frac{2^K}{2^{H(\frac{\delta-1}{K})K}} \le\frac{2^K}{\sum_{i=0}^{\delta-1}\binom{K}{i}}\le A(K, \delta)\\
\end{aligned}
\end{equation}
whence $H(\frac{\delta-1}{K})\leq 1-\frac{H(\pmb{L})}{K}$.
\end{proof}

In experiments, we can estimate
$H(\pmb{L})$$=H(\pmb{L}_1,$$\pmb{L}_2,$$\dots,$$\pmb{L}_C)$ $=$
$\sum_{i=1}^{C}H(\pmb{L}_i)$ by assuming independence between tags
for simplicity and each tag $\pmb{L}_i$ follows a Bernoulli
distribution, i.e., $\pmb{L}_i\sim B(\theta_i)$. The estimation of
$\theta_i$ can be
$\hat{\theta_i}=\frac{1}{N}\sum_{n=1}^N\mathbb{I}(l_{n,i}=1)$, $l\in
\mathbf{L}$, where $\mathbb{I}\{*\}$ is indicator function. Then,
given code length $K$ and $\delta$, we can investigate whether the
value of $\delta$ is larger than the effective range. In our
experiments, we found that the upper-bound of the effective $\delta$
is always less than $K/2$, which is preferred set by previous
works\footnote{The bit balance, firing each bit as 1 or -1 with
50\%, is the special case of $\delta=K/2$, which can be derived by
the Maximum Entropy Principle but ignores the limited coding ability
of codes.}.

\noindent {\bf 2) Lower bound.} Let $H_*$ denote $H(\pmb{S_{*,:}})$,
by the fact (2) we have $\delta\ge \mathop{\max}_{*=1:N}\{H_*\}$.
However, as $S$ is sparse, we could relax  $\delta\ge
\mathop{\max}_{*=1:N}\{H_*\}$ by making the $\delta$ be larger than
most $H_*$ to ensure $\delta$ bits can encode the semantic neighbors
of each sample with a certain probability. By Chebyshev's
Inequality, we have:
\begin{equation}
\label{eq:lower_bound1}
\begin{aligned}
P\{H_*\le\delta\}\ge 1-\frac{D(H_*)}{(\delta-E(H_*))^2}
 \end{aligned}
\end{equation}
where $E(H_*)=\frac{1}{N}\sum_{i=1}^{N}H_i$,
$D(H_*)=\frac{1}{N}\sum_{i=1}^{N}(H_i-E(H_*))^2$. If
$P\{H_*\le\delta\}= p$, $1>p>1/2$, we have:
\begin{equation}
\label{eq:lower_bound}
\begin{aligned}
\delta \ge \sqrt{\frac{D(H_*)}{1-p}}+E(H_*)
 \end{aligned}
\end{equation}
We estimate $H_i=-\sum_{j=1}^{|l_i|}p(s_{i}=\frac{j}{|l_i|})\log
p(s_i=\frac{j}{|l_i|})\le|l_i|$,
$p(s_i=\frac{j}{|l_i|})=\frac{\sum_*\mathbb{I}(S_{i,*}=j/|l_i|)}{\sum_*
\mathbb{I}(|l_i\cap l_*|=S_{i,*}\cdot|l_i|)}$, $s_i\in S_{i,:}$. For
simple computation, we substitute $H_i$ with $|l_i|$ in experiments.

Combing Eqs.~(\ref{eq:upper_bound}) and (\ref{eq:lower_bound}), we
obtain the final bounds of effective robust parameter $\delta$. In
practice, we can pre-determine the range of high quality $\delta$ by
substituting the corresponding value into Eq.~(\ref{eq:upper_bound})
and (\ref{eq:lower_bound}) when given training set information and
avoid the burden of seeking proper value for effective learning. It
is worth note that the performance of different $\delta$ within
bounds is not significantly different, setting any value in bounds
is OK in practice. Furthermore, if we want to find optimal $\delta$,
our theoretical results can guarantee the performances within bound
are better than those without bound, which significantly reduces the
search range of optimal hyper-parameter. Therefore, setting the
special value within the bound is not difficult. The experimental
results in Sec.~\ref{sec:exp3} demonstrate the effectiveness of the
derived bounds.

\subsection{Network Architecture}
The architecture of RMSH is illustrated in Fig.~\ref{fig:method},
which seamlessly integrates three parts: the image hashing, the text
hashing, and the pseudo-codes networks.

{\bf Deep Hashing Networks.} We build image and text hashing as two
deep neural networks via adding two fully-connected layers with tanh
function on the top feature layer of commonly-used modality-specific
deep models, e.g., CNN model ResNet for image modality, BOW or
sentence2vector for text modality. These two layers as the hash
functions transform the feature $x_i$, $y_i$ into binary-like codes
$z_i^{(x)}$, $z_i^{(y)}$$\in$$\mathbb{R}^K$. Then, we obtain the
hash codes of modality $*$ by $b_i^{(*)}=sgn(z_i^{(*)})$, where
$sgn(z)$ is the sign function.

{\bf Pseudo-codes Network (PCN).} The input data are always
imbalanced due to the sparsity problem of multi-labels, which
results in the difficulty of capturing complicated similarity
structure across modality. To handle this problem, inspired by
work~\cite{AlfassyKASHFGB19}, we propose pseudo-codes networks to
manipulate the binary-like codes at the semantic level for
generating codes of rare semantics, which are mixed with original
codes to explore complicated similarity structure. We defined two
types of code operations, i.e., union and intersection, as two
fully-connected networks:
\begin{equation}
\begin{aligned}
  z_{1\oplus2}&= f_u(z_1,z_2) \equiv tanh(W_u[z_1,z_2])\\
  z_{1\otimes2}&= f_t(z_1,z_2) \equiv tanh(W_t[z_1,z_2])\\
\end{aligned}
\end{equation}
where $[\cdot,\cdot]$ denotes the concatenation operation. $W_u,W_t$
$\in \mathbb{R}^{K\times 2K}$ are weights to be learnt.
$z_{1\oplus2}=f_{u}(z_1,z_2)$ with label $l_u=l_1\oplus l_2=l_1\cup
l_2$. $z_{1\otimes2}=f_{t}(z_1,z_2)$ with label $l_{t}=l_1\otimes
l_2$, defined as $l_1\cap l_2$.

\subsection{Objective Function}
To preserve multilevel semantic similarity and fully exploit the
space between dissimilar points in Hamming space, we formulate the
goal in Sec~\ref{sec:problem} as follows: for $\forall$ $b_i$,
$b_j$, $b_k$, 1) if $S_{i,j}$$>$$S_{i,k}$$>0$, then
$d_H(b_i,b_k)$$-$$d_H(b_i,b_j)$$\ge$$(S_{i,j}$$-$$S_{i,k})$$\cdot\delta$.
2) if $S_{i,j}$$=$$0$, then $d_H(b_i,b_j)$$\ge$$\delta$. The first
goal is to preserve the ranking information of similar points by an
adaptive margin. The second goal is to keep the minimum distance
between dissimilar points. We couple these two goals in a scheme of
margin-adaptive triplets, that is:
\begin{equation}
\label{eq:triplet_definition}
\begin{aligned}
 \mathcal{L}_{t}(b_i,b_j,b_k)&=y_{i,j}\cdot y_{i,k}\cdot[d_H(b_i,b_j)-d_H(b_i,b_k)+\alpha]_{+}\\
 &+(1-y_{i,j})[\delta-d_H(b_i,b_j)]_{+}+(1-y_{i,k})[\delta-d_H(b_i,b_k)]_{+}\\
\end{aligned}
\end{equation}

where $[\cdot]_+=max\{0,\cdot\}$, $y_{i,j}=\mathbb{I}(S_{i,j}>0)$,
$y_{i,k}=\mathbb{I}(S_{i,k}>0)$, $\mathbb{I}$ is indication
function, $\alpha=\delta(S_{i,j}-S_{i,k})$, is adaptive to the
discrepancy of similarity level and $\delta$ controls the distance
of dissimilar points. Given a triplet of multi-modal data
$x_{\{1,2,3\}}$, $y_{\{1,2,3\}}$ and their labels $l_{\{1,2,3\}}$,
we can obtain five hash codes for each modality, i.e.,
$b_{\{1,2,3\}}$, and $b_4=sgn(f_u(z_1,z_2))$,
$b_5=sgn(f_t(z_1,z_2))$. Note that in principle we can generate more
pseudo-codes using the PCN network but for simplicity and w.l.o.g.,
here only a set involving two codes are used. We partition them into
the intra-modality triplets $b_{\{1,2,3\}}$, $b_{\{1,2,4\}}$,
$b_{\{1,2,5\}}$, and the inter-modality triplets $\{b_1^{(y)},
b_{\{2,3\}}^{(x)}\}$, $\{b_2^{(y)}, b_{\{1,3\}}^{(x)}\}$,
$\{b_3^{(y)}, b_{\{1,2\}}^{(x)}\}$, each of which reflects some
aspect of the similarity structure and corresponds to a separate
loss term in the final loss.

To mine the effective information in each triplets quickly, we
select one point $b_{*}$ from each triplet with most tags as the
reference point and re-order other points as
$\{b_*,b_{*,1},b_{*,2}\}$, where $b_{*,1}$ is the one that is more
similar to $b_{*}$ than the other $b_{*,2}$. Then, we substitute
them into Eq.~(\ref{eq:triplet_definition}) and obtain the used
triplet loss as
$\mathcal{L}_{t}(b_{\{1,2,3\}})=\mathcal{L}_{t}(b_*,b_{*,1},b_{*,2})$,
where $\alpha=\frac{|l_*\cap l_{*,1}|-|l_*\cap
l_{*,2}|}{|l_*|}\cdot\delta$.

The final triplet loss for one modality is defined as:
\begin{small}
\begin{equation}
\label{eq:triplet_loss}
\begin{aligned}
 \mathcal{L}_{tr}(&b_{\{1,2,3\}}^{(x)})= \lambda_1(\sum_{i=3}^{5}\mathcal{L}_{t}(b^{(x)}_{\{1,2,i\}}))+\lambda_2(\mathcal{L}_{t}(b_1^{(y)},b_{\{2,3\}}^{(x)})\\
 &+\mathcal{L}_{t}(b_2^{(y)},b_{\{1,3\}}^{(x)})+\mathcal{L}_{t}(b_3^{(y)},b_{\{1,2\}}^{(x)}))\\
 \end{aligned}
\end{equation}
\end{small}
where $\lambda_1$, $\lambda_2$ control the weight of intra-modal and
inter-modal similarity loss, respectively.

Besides, we adopt weighted cross-entropy loss to ensure that each
individual hashing codes consistent with its own semantics,
especially for the pseudo-codes from fusion, and the loss is defined
as follows:
\begin{small}
\begin{equation}
\label{eq:classification_loss}
\begin{aligned}
 \mathcal{L}_{cl}(b)=-\sum_{j=1}^{C}\big(w_{p}\cdot l_j\log(\hat{l_j})+(1-l_j)\log(1-\hat{l_j}) \big)
\end{aligned}
\end{equation}
\end{small}
where $\hat{l}$ is the predicted value of RMSH, and $w_{p}$ is the
weight of positive samples.

Finally, the overall objective for N training triplets $\{b_{i1},$
$b_{i2},$$b_{i3}\}_{i=1}^{N}$ of RMSH is defined as follows:
\begin{equation}
\label{eq:final_loss}
\begin{aligned}
 \mathop{min}_{\theta}\mathcal{L}&=\sum_{i=1}^{N}\big(\sum_{j=1}^{3}\mathcal{L}_{cl}(b^{(x,y)}_{ij})+\mathcal{L}_{tr}(b^{(x,y)}_{i\{1,2,3\}})\\
 &+\lambda_{3}(\mathcal{L}_{cl}(b^{(x,y)}_{i4})+\mathcal{L}_{cl}(b^{(x,y)}_{i5}))\big)\\
 &s.t.\ b^{(x)},b^{(y)}\in \{-1,1\}^{K}\\
 \end{aligned}
\end{equation}
where $\lambda_3$ balances the classification loss of pseudo-codes,
$ \theta$ is the parameter of image hashing, text hashing, and PCN.
\subsection{Optimization}
Since the Eq.~(\ref{eq:final_loss}) is a discrete optimization
problem, we follow the work~\cite{Wu2017Deep} to solve this problem.
We relax $b$ as $z$ by introducing quantization error $||z-b||_2^2$,
and substitute Hamming distance with Euclidean distance, i.e.
$d(z_1,z_2)=||z_1-z_2||_{2}^{2}$. The Eq.~(\ref{eq:final_loss}) is
rewritten as follows:
\begin{equation}
\label{eq:final_loss1}
\begin{aligned}
 \mathop{min}_{\theta,b}\mathcal{L}&=\sum_{i=1}^{N}\big(\sum_{j=1}^{3}L_{cl}(z^{(x,y)}_{ij})+\lambda_{3}(\sum_{j=4}^{5}L_{cl}(z^{(x,y)}_{ij}))\\
 &+L_{tr}(z^{(x,y)}_{i\{1,2,3\}})+\lambda_4\sum_{j=1}^{3}||z^{(x,y)}_{ij}-b_{ij}||_2^2\big)\\
&s.t.\ b \in \{-1,1\}^{K}\\
 \end{aligned}
\end{equation}
where $z^{(*)}_{i4}=f_u(z^{(*)}_{i1},z^{(*)}_{i2})$,
$z^{(*)}_{i5}=f_t(z^{(*)}_{i1},z^{(*)}_{i2})$. $\lambda_{4}$
controls the weight of quantization error. In the training phase, we
let $b_i$ $=b_i^{(x)}$ $=b_i^{(y)}$ for better performance. We adopt
an alternating optimization strategy to learn $\theta$ and $b$ as
follows.

{\bf 1) Learn $\theta$ with $b$ fixed.} When $b$ is fixed, the
optimization for $\theta$ is performed using stochastic gradient
descent based on Adaptive Moment Estimation (Adam).

{\bf 2) Learn $b$ with  $\theta$ fixed.} When $\theta$ is fixed, the
problem in Eq.~(\ref{eq:final_loss1}) can be reformulated as
follows:
\begin{equation}
\label{eq:final_loss2}
\begin{aligned}
 \mathop{max}_{b}\ &\lambda_4\sum_{i=1}^{N}\big(b_i^T(z^{(x)}_{i}+z^{(y)}_{i})\big)\\
 &s.t.\ b \in \{-1,1\}^{K}\\
 \end{aligned}
\end{equation}
We can find that the binary code $b_i$ should keep the same sign as
$z^{(x)}_{i}+z^{(y)}_{i}$, i.e., $b_i=sgn(z^{(x)}_{i}+z^{(y)}_{i})$.

\section{Experiments and Discussions}
\label{sec:exp}
\subsection{Datasets}
{\bf Microsoft COCO}~\cite{Lin2014Microsoft} contains 82,783
training and 40,504 testing images. Each image is associated with
five sentences (only the first sentence is used in our experiments),
belonging to 80 most frequent categories. We use all training set
for training and sample 4,956 pairs from the testing set as queries.

{\bf MIRFLICKR25K}~\cite{HuiskesL08mir} contains 25,000 image-text
pairs. Each point associates with some of 24 labels. We remove pairs
without textual tags or labels and subsequently get 18,006 pairs as
the training set and 2,000 pairs as the testing set. The
1386-dimensional bag-of-words vector gives the text description.

{\bf NUS-WIDE}~\cite{civr09nuswide} contains 260,648 web images,
belonging to 81 concepts. After pruning the data without any label
or tag information, only the top 10 most frequent labels and the
corresponding 186,577 text-image pairs are kept. 80,000 pairs and
2,000 pairs are sampled as the training and testing sets,
respectively. The 1000-dimensional bag-of-words vector gives the
text description. We sampled 5,000 pairs of the training set for
training.

\begin{table*}
\caption{Comparison (NDCG@500) of different cross-modal hashing
methods on three datasets with different code length. The best
result is shown in boldface.} \label{tab:tab_compare_cross_hash}
\resizebox{\textwidth}{!}{%
\begin{tabular}{|c|c|c|c|c|c|c|c|c|c|c|c|c|c|}
\hline \multirow{2}{*}{Task} & \multirow{2}{*}{Method} &
\multicolumn{4}{c|}{MIRFLICKR25K} & \multicolumn{4}{c|}{NUS-WIDE} &
\multicolumn{4}{c|}{MS COCO} \\ \cline{3-14}
 &  & 16bits & 32bits & 64bits & 128bits & 16bits & 32bits & 64bits & 128bits & 16bits & 32bits & 64bits & 128bits \\ \hline
\multirow{8}{*}{I $\to$ T}& CMFH & 0.2926 & 0.3154 & 0.3193 & 0.3282 & 0.4875 & 0.5012 & 0.5270 & 0.5394 & 0.2224 & 0.2571 & 0.2899 & 0.3098 \\
 & UCH & 0.3017& 0.3114&0.3197 &0.3362 &0.5202 &0.5175  &0.5227 & 0.5175&0.2239 &0.2588 &0.2757  & 0.3143 \\
 & SePH & 0.3310 & 0.3522 & 0.3667 & 0.3811 &0.5726 & 0.5809 &0.5867  &0.6043  & 0.2266 & 0.2681 & 0.3113 & 0.3375 \\
 & DCMH & 0.3857 & 0.4013 & 0.4162 & 0.4046 & 0.5942 & 0.6239 & 0.6353 & 0.6392 & 0.2257 & 0.2661& 0.3089 & 0.3147 \\
 & TDH & 0.4379 & 0.4491 & 0.4491 & 0.4458 & 0.6196 & 0.6242 & 0.6305 & 0.6324 &0.2447 & 0.2692 &0.2946  &0.3138  \\
 & SSAH & 0.4203 & 0.4392 & 0.4626 & 0.4753 & 0.6026 & 0.6331 & 0.6263 & 0.6144 & 0.1955 & 0.2797 & 0.3157 & 0.3917 \\
 & SVHN & 0.4186 & 0.4354 & 0.4675 & 0.4895 & 0.6039 & 0.6324 & 0.6486 & 0.6501 &0.2526  & 0.2930 & 0.3270 & 0.3811 \\ \cline{2-14}
 & RMSH & \textbf{0.4592} & \textbf{0.4940} & \textbf{0.5180} & \textbf{0.5183} & \textbf{0.6148} & \textbf{0.6401} & \textbf{0.6497} & \textbf{0.6574} & \textbf{0.3037} & \textbf{0.3724} & \textbf{0.3855} & \textbf{0.3998} \\ \hline
\multirow{8}{*}{T $\to$ I} & CMFH & 0.2857 & 0.3079 & 0.3115 & 0.3138 & 0.4642 & 0.4775 & 0.4998 & 0.5091 & 0.1982 & 0.2182 & 0.2398 & 0.2539 \\
 & UCH & 0.3027 &0.3000 &0.3157  &0.3051  & 0.4923 &0.5103 &0.5008 &0.5232 & 0.1895& 0.2356 &0.2569  & 0.2687\\
 & SePH & 0.3085 & 0.3245 & 0.3168 & 0.3582 & 0.5264 & 0.5285 & 0.5245 &0.5337  & 0.1968 & 0.2367 & 0.2627 & 0.2839 \\
 & DCMH & 0.3719 & 0.3819 & 0.3877 & 0.3894 & 0.5375 & 0.5451 & 0.5438 & 0.5463 & 0.1885 & 0.2122 & 0.2544& 0.2601 \\
 & TDH & 0.3789 & 0.4084 & 0.3924 & 0.3937  & 0.5329 &0.5340  & 0.5368 & 0.5306 & 0.2249 & 0.2420 & 0.2521 &0.2712 \\
 & SSAH & 0.3648 & 0.3815 & 0.3710 & 0.3707 & 0.5262 & 0.5428 & 0.5583 & 0.5375 & 0.1870 & 0.2382 & 0.2469 & 0.3114 \\
 & SVHN & 0.3985 & 0.4434 & 0.4483 & 0.4604 & 0.5486 & 0.5702 & 0.5831 & 0.5885 & 0.2313 & 0.2664 & 0.2780 & 0.3163 \\ \cline{2-14}
 & RMSH & \textbf{0.4379} & \textbf{0.4533} & \textbf{0.4562} & \textbf{0.4628} & \textbf{0.5744} & \textbf{0.5921} & \textbf{0.5900} & \textbf{0.5999} & \textbf{0.2897} & \textbf{0.3080} & \textbf{0.3050} & \textbf{0.3175} \\ \hline

\end{tabular}%
}
\end{table*}

\subsection{Evaluation protocol and Baselines}
{\bf Evaluation protocol.} We perform cross-modal retrieval with two
tasks. (1) {\bf Image to Text (I $\to$ T)}: retrieve relevant data
in the text training set using a query in the image test set. (2)
{\bf Text to Image (T $\to$ I)}: retrieve relevant data in the image
training set using a query in the text test set. Particularly, to
evaluate the retrieval performance of multi-label data, we adopt the
Normalized Discounted Cumulative Gain (NDCG)~\cite{Kalervo2000IR} as
the performance metric, which is defined as follows, since the NDCG
can evaluate the ranking by penalizing errors according to the
relevant score and position.
\begin{equation}
\label{eq_ndcg_metric}
  NDCG@p=\frac{1}{Z}\sum_{i=1}^{p}\frac{2^{r_i}-1}{\log(1+i)}
\end{equation}
where Z is the ideal $DCG@p$ and calculated form the correct ranking
list, and $p$ is the length of list. $r(i)=|l_q\cap l_i |$ denotes
the similarity between the $i$-th point and the query, following the
definition in work~\cite{zhao2015deep}. $l_q$ and $l_i$ denote the
label of the query and $i$-th position point.

Besides, the commonly-used precision-recall curve is used as a
metric to evaluate the performance, where we consider two points are
relevant if they share at least one tag.

\noindent {\bf Baselines.} We compare our proposed RMSH with six
cross-modal hashing methods {\bf CMFH}~\cite{Ding2014Collective},
{\bf UCH}~\cite{li2019coupled}, {\bf SePH}~\cite{Lin2015Semantics},
{\bf DCMH}~\cite{Jiang2017Deep}, {\bf TDH}~\cite{DengCLGT18}, {\bf
SSAH}~\cite{LiDL0GT18}, {\bf SVHN}~\cite{HuWZP19}. For a fair
comparison, all shallow methods take the deep off-the-shelf features
extracted from ResNet-152~\cite{He2015Deep} pre-trained on ImageNet
as inputs, and all deep models are implemented carefully with the
same CNN sub-structures for image data and the same multiple
fully-connect layers for textual data. The hyper-parameters of all
baselines are set according to the original papers or experimental
validations.

\subsection{Implementation details.} Our RMSH method is
implemented with Tensorflow\footnote{https://www.tensorflow.org/}.
We use ResNet~\cite{He2015Deep} pre-trained on ImageNet as the CNN
model in image hashing to extract features of image modality. For
the MS COCO dataset, the 4800-dimensional Skip-thought
vector~\cite{KirosZSZUTF15} gives the sentence description. The
structure of two hashing layers is $D^{(*)}\to 1024\to K$-bit
length, where $D$ is the dimension of * modality's feature. For fast
training, we only optimize the hashing layers with the pre-trained
CNN sub-structures (no fine-tuned on the target dataset for a fair
comparison with other shallow methods) fixed for all deep methods.
In all experiments, we set the batch size to 128, epochs to 50,
$\lambda_1$, $\lambda_2$, $\lambda_3$, and $\lambda_4$ to 0.01, 0.1,
0.1, 0.1, learning rate to 0.001, $\delta$ by the derived results in
Sec.~\ref{sec:emm}, $w_p$ to 20 for NUS-WIDE and MIRFLICKR25K
datasets, 30 for MS COCO dataset.

\begin{figure*}
\center
\begin{tabular}{ccc}
\includegraphics[scale=.45]{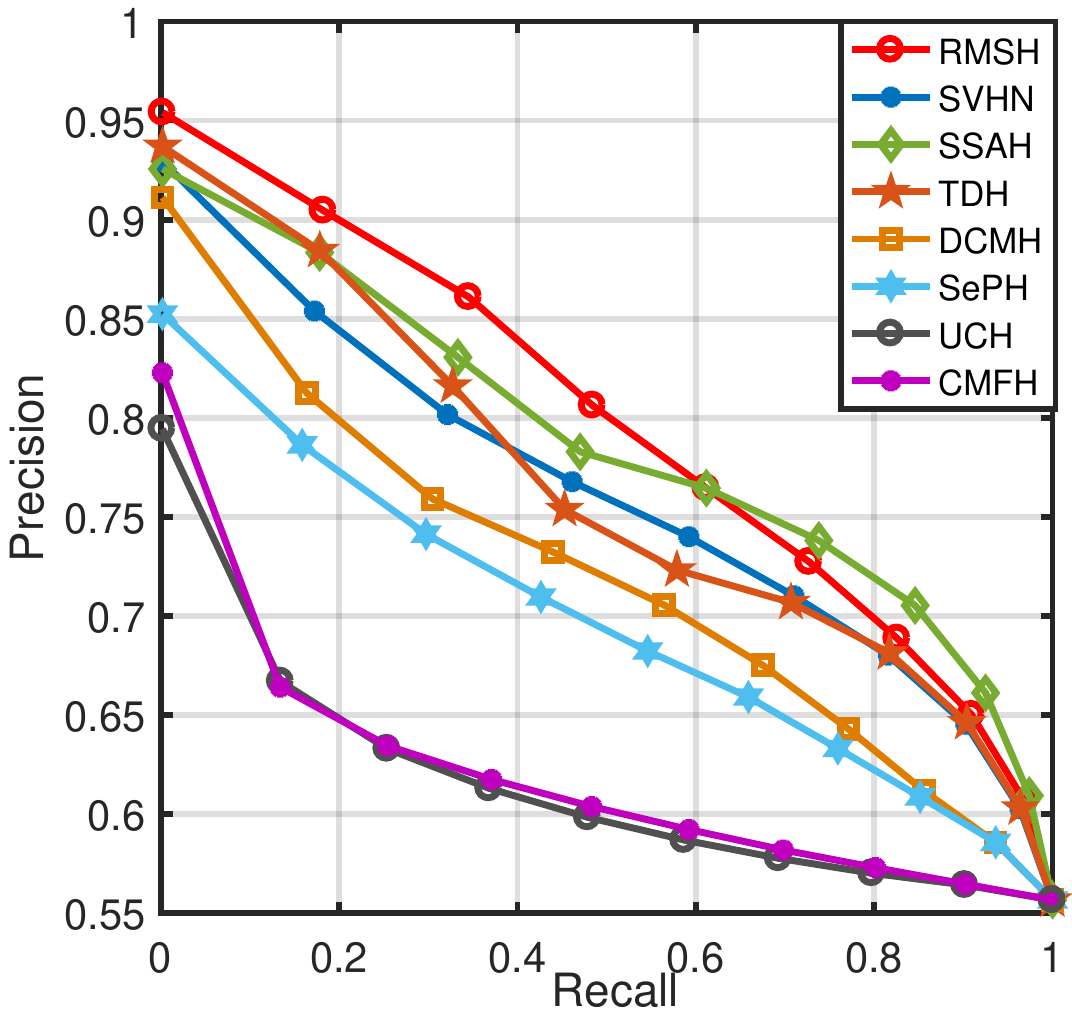}&
\includegraphics[scale=.45]{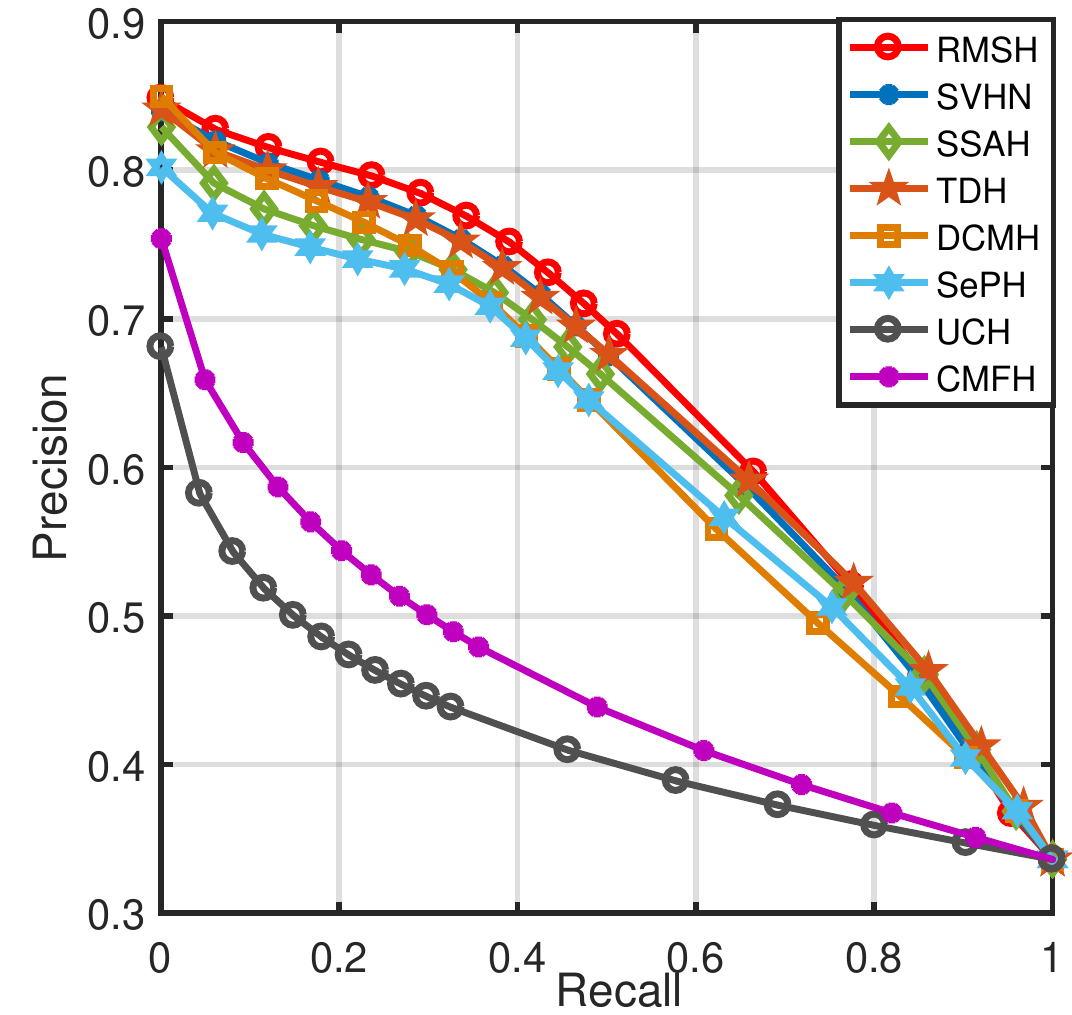}&
\includegraphics[scale=.45]{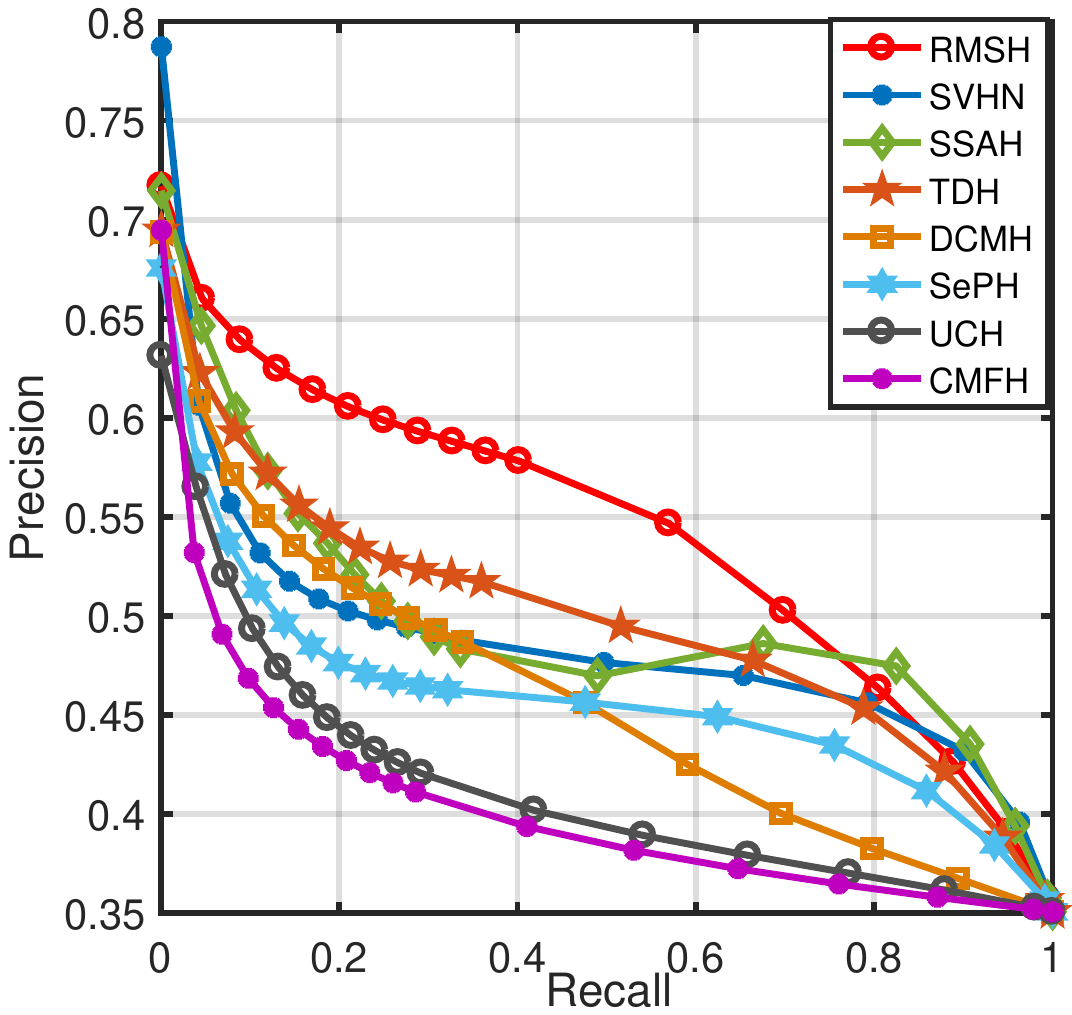}\\
@MIRFLICKR25K (I$\to$T)& @NUS-WIDE (I$\to$T) & @MS COCO (I$\to$T)\\
\includegraphics[scale=.45]{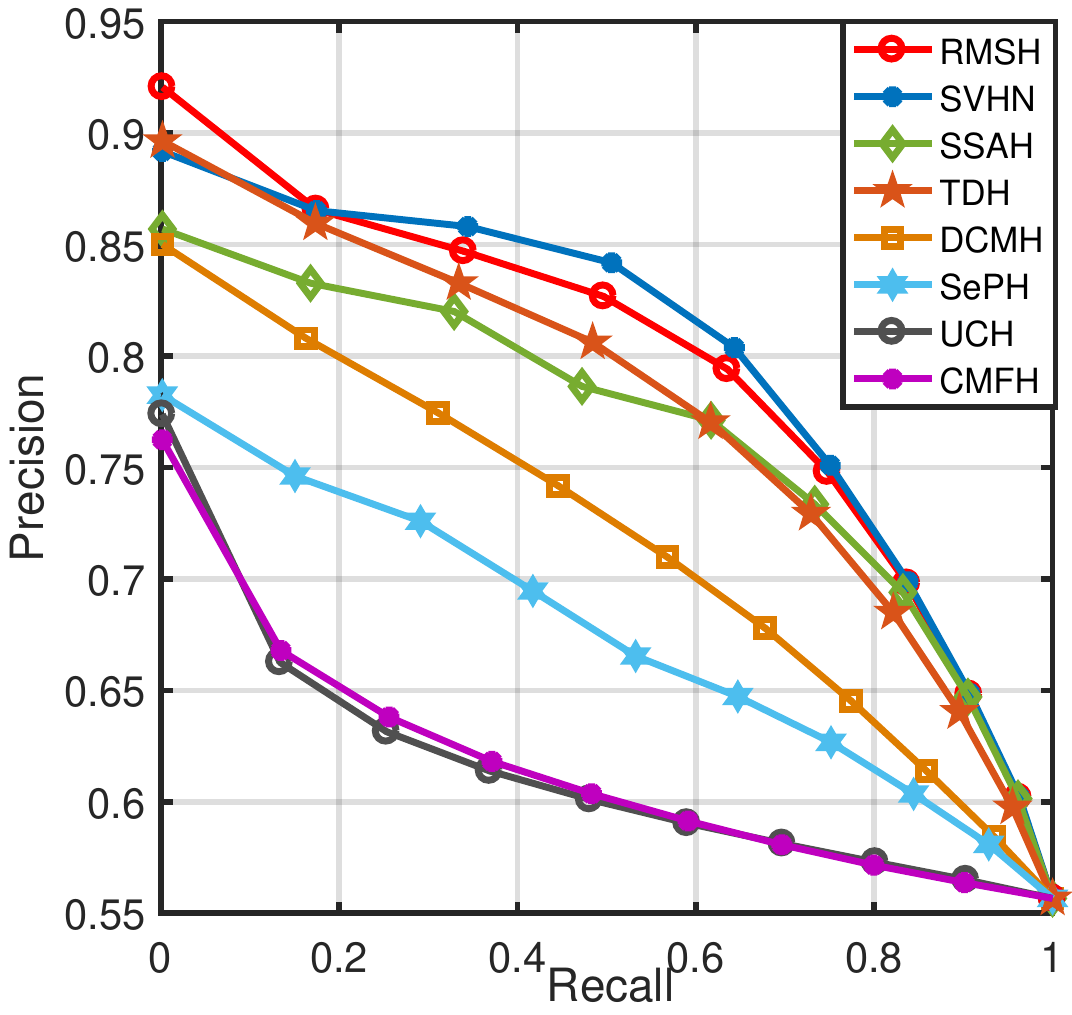}&
\includegraphics[scale=.45]{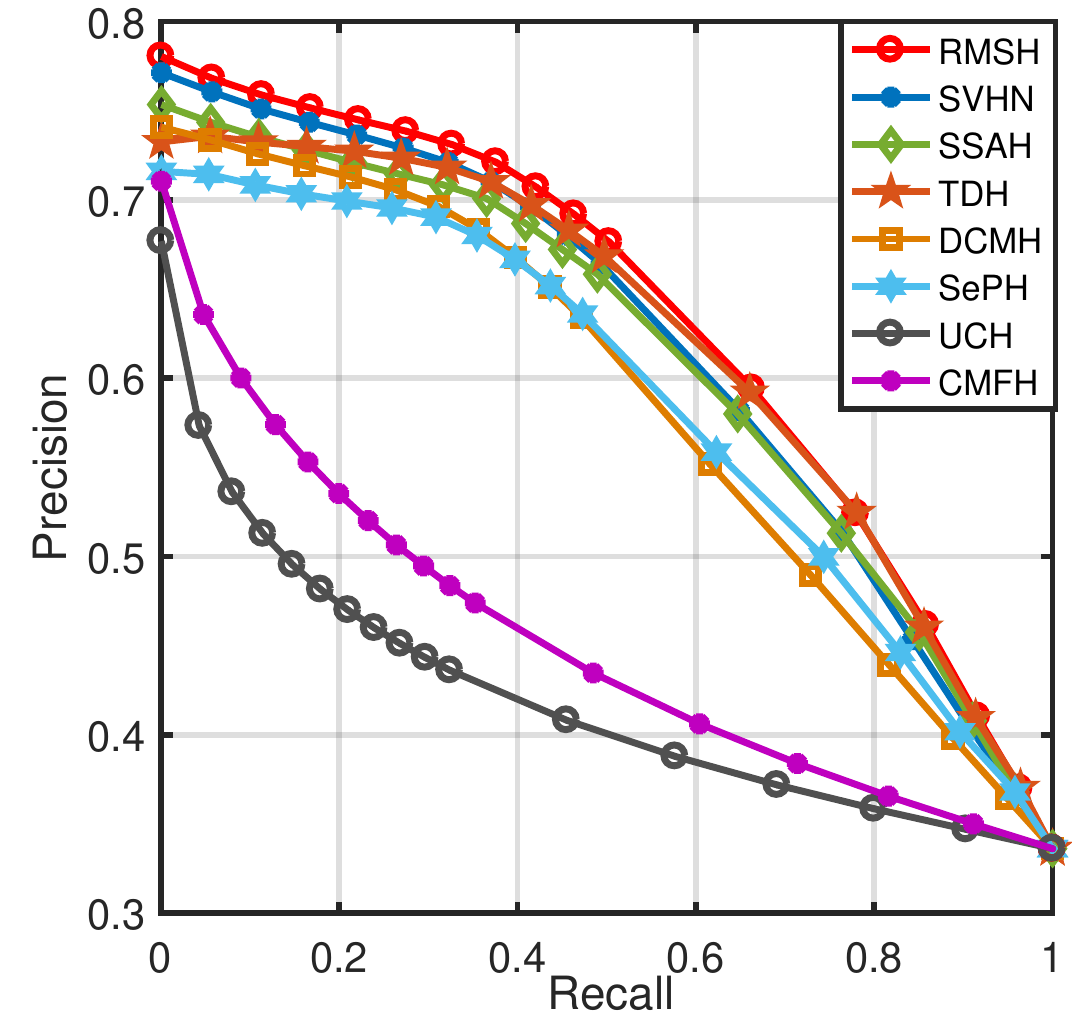}&
\includegraphics[scale=.45]{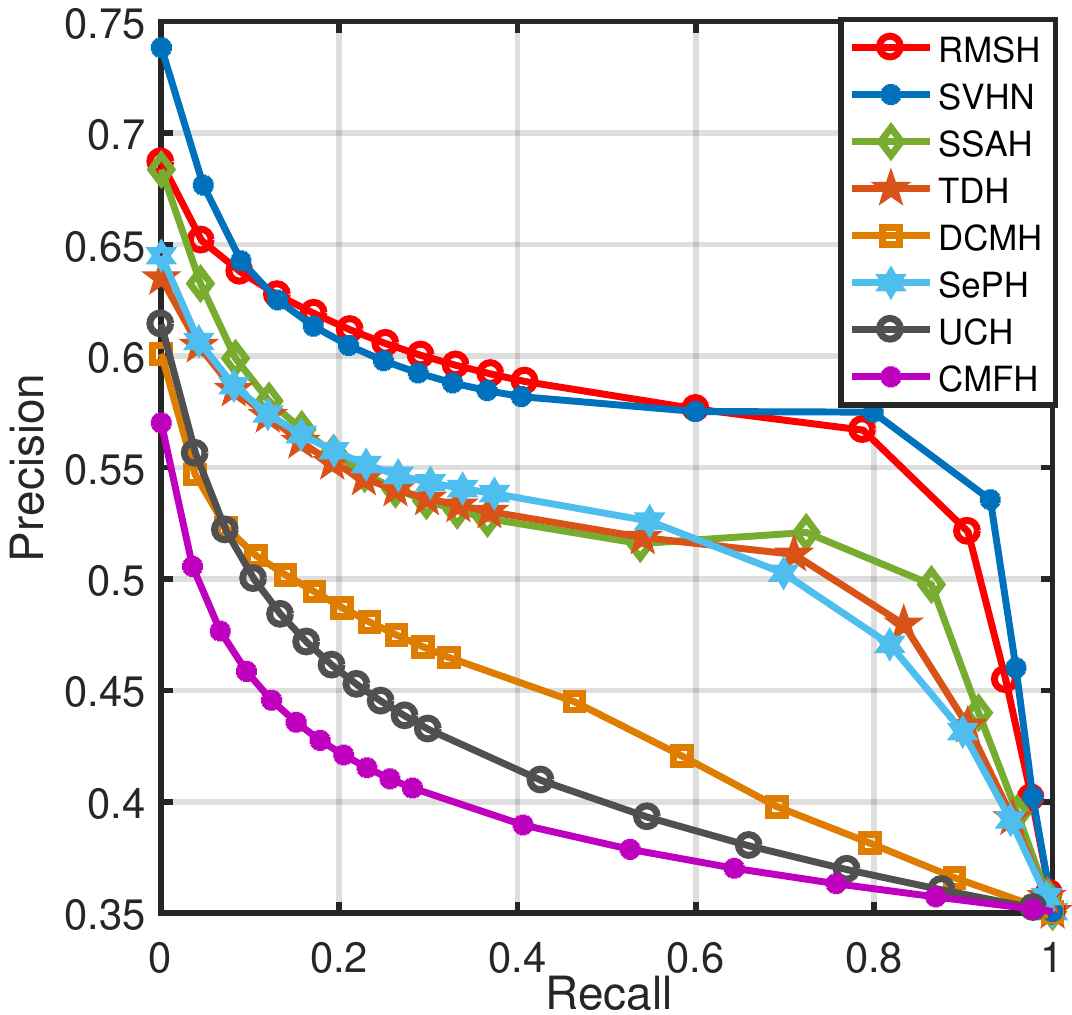}\\
@MIRFLICKR25K (T$\to$I)& @NUS-WIDE (T$\to$I)& @MS COCO (T$\to$I)\\
\end{tabular}
\caption{Precision-Recall curves on MIRFLICKR25K, NUS-WIDE, and MS
COCO datasets.} \label{fig:campare_of_PR}
\end{figure*}

\begin{figure}
\center
\includegraphics[scale=0.35]{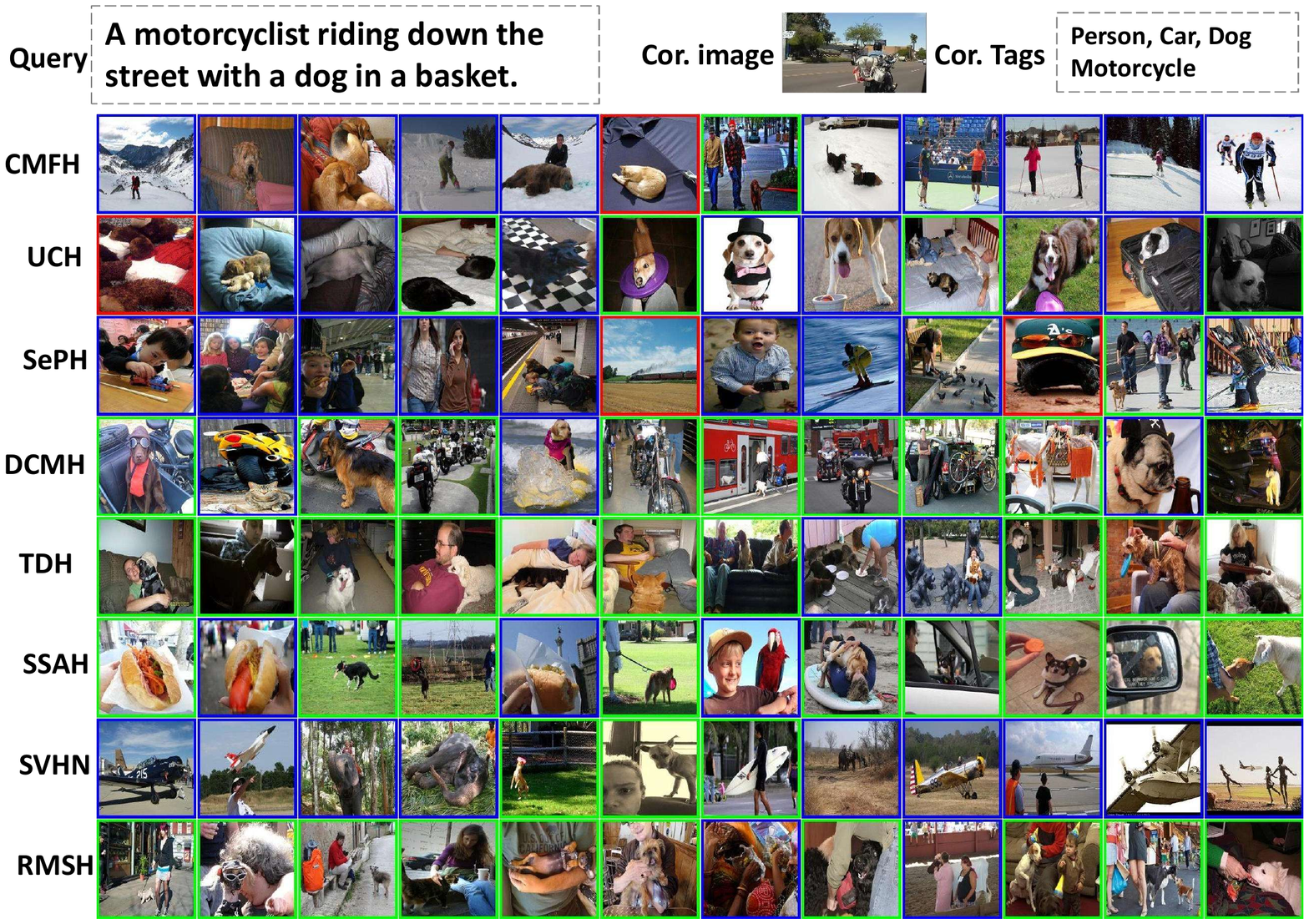}\\
\caption{Retrieval examples of 'Text to Image'. Red border denotes
irrelevant; blue denotes sharing one tag with the query, green
denotes sharing two tags with the query at least.}
\label{fig:retrieval_example_coco_TI}
\end{figure}

\begin{figure}
\center
\includegraphics[scale=0.35]{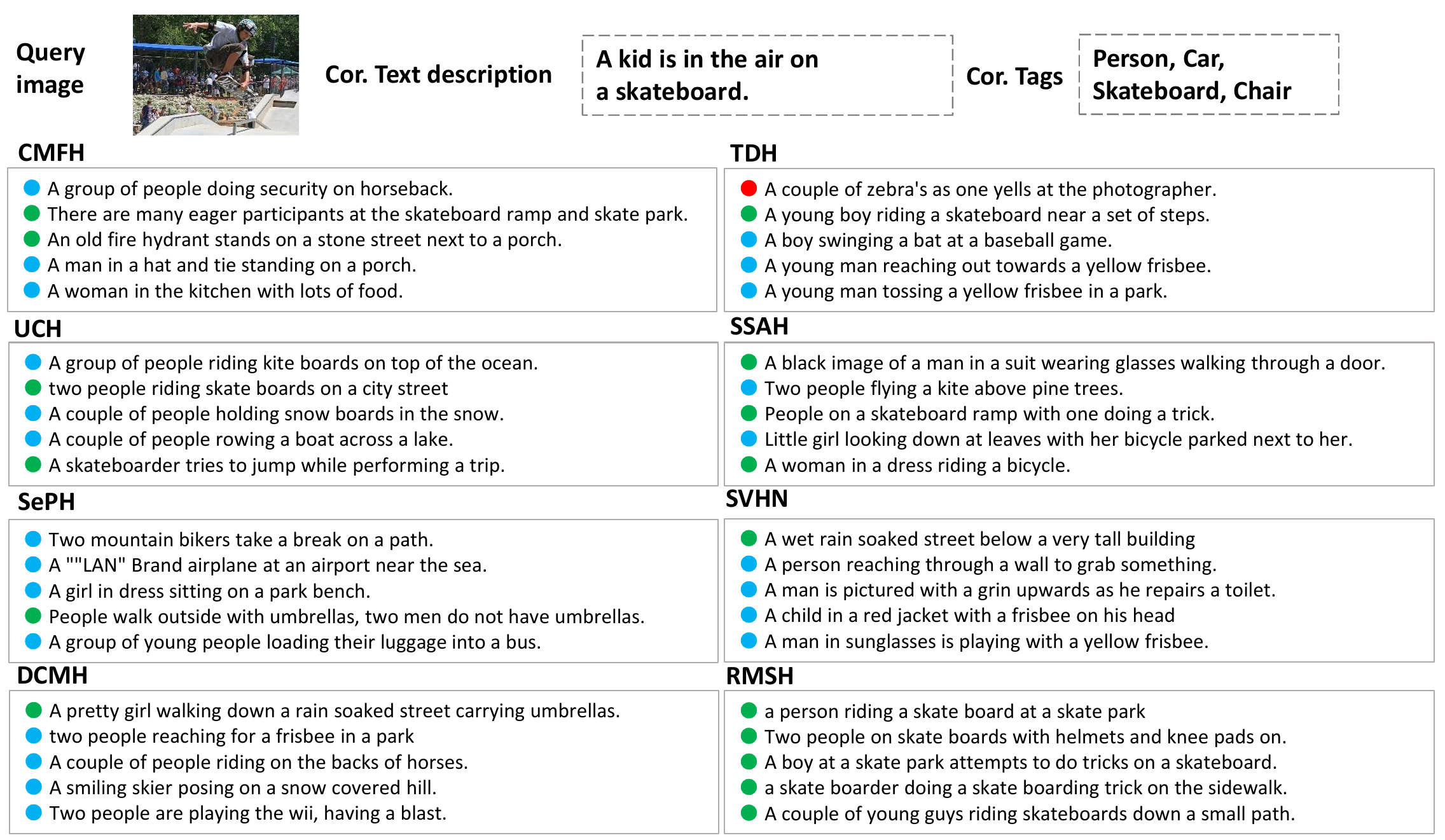}
\caption{Retrieval examples of 'Image to Text'. Red dot denotes
irrelevant; blue denotes sharing one tag with the query, green
denotes sharing two tags with the query at least.}
\label{fig:retrieval_example_coco_IT}
\end{figure}

\begin{figure}[t]
\center
\begin{tabular}{cc}
\includegraphics[scale=.4]{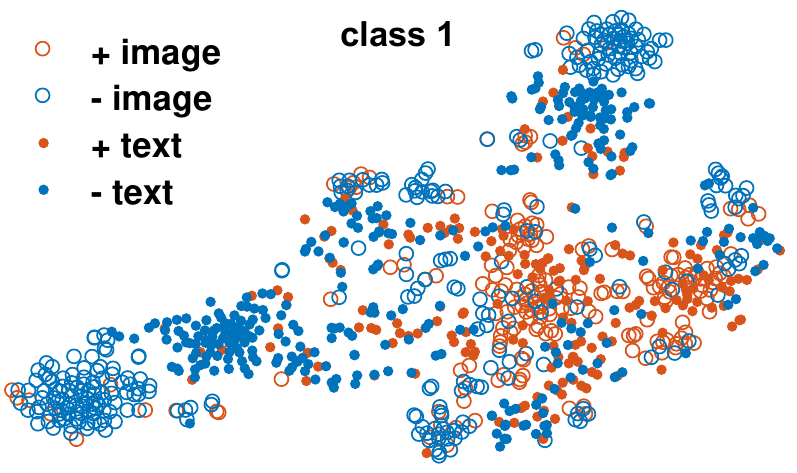}&
\includegraphics[scale=.4]{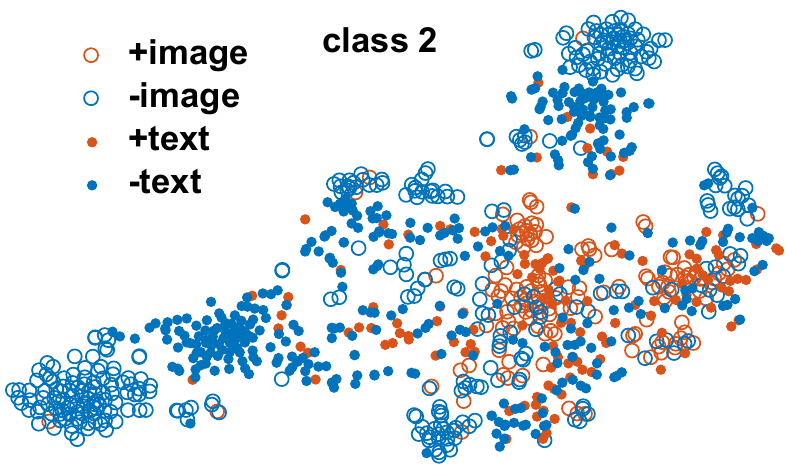}\\
\includegraphics[scale=.4]{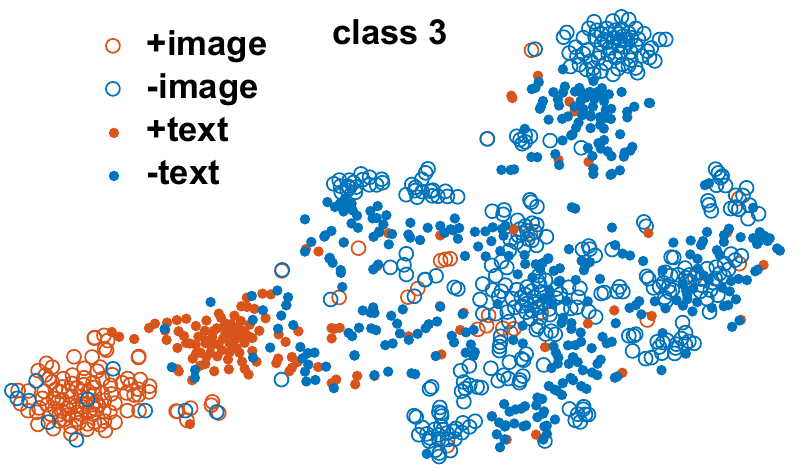}&
\includegraphics[scale=.4]{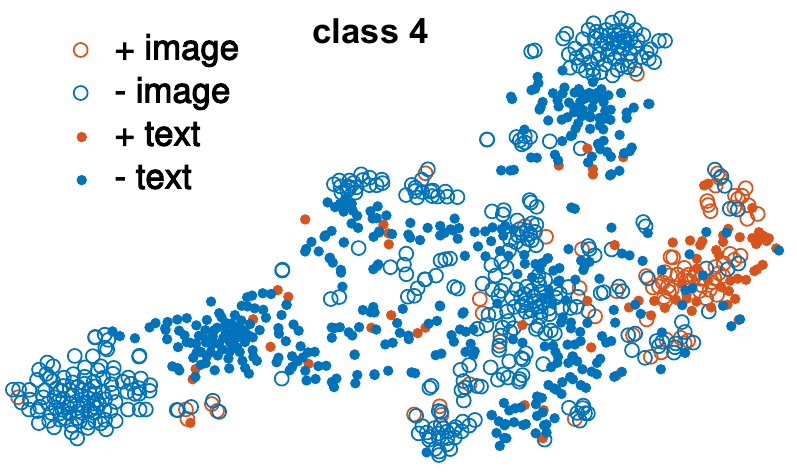}\\
\includegraphics[scale=.4]{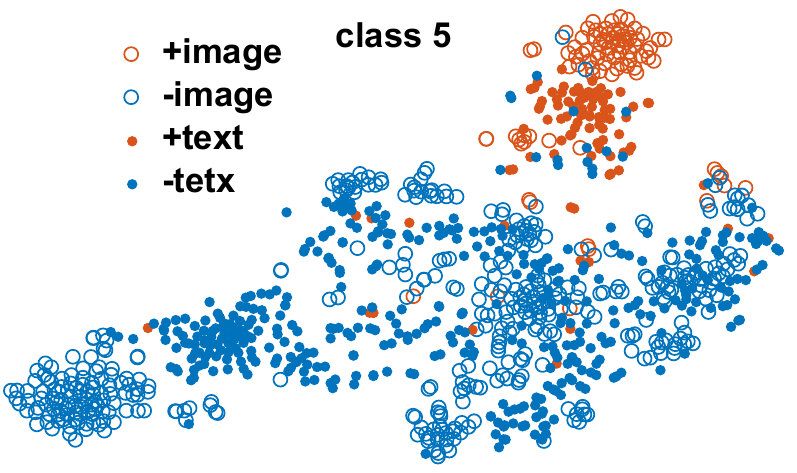}&
\includegraphics[scale=.4]{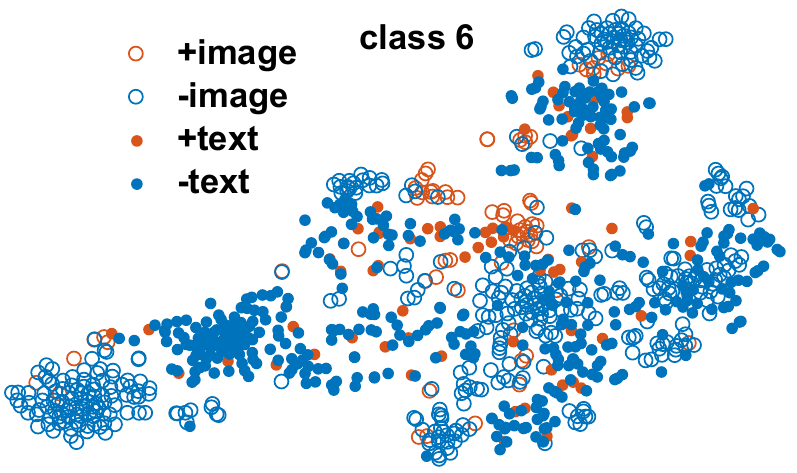}\\
\end{tabular}
\caption{The visualization of binary-like codes of testing image and
text learned by RMSH on the NUS-WIDE (six classes) by t-SNE.}
\label{fig:visualize_distribution}
\end{figure}

\begin{figure*}
\center
\begin{tabular}{ccc}
\includegraphics[scale=.45]{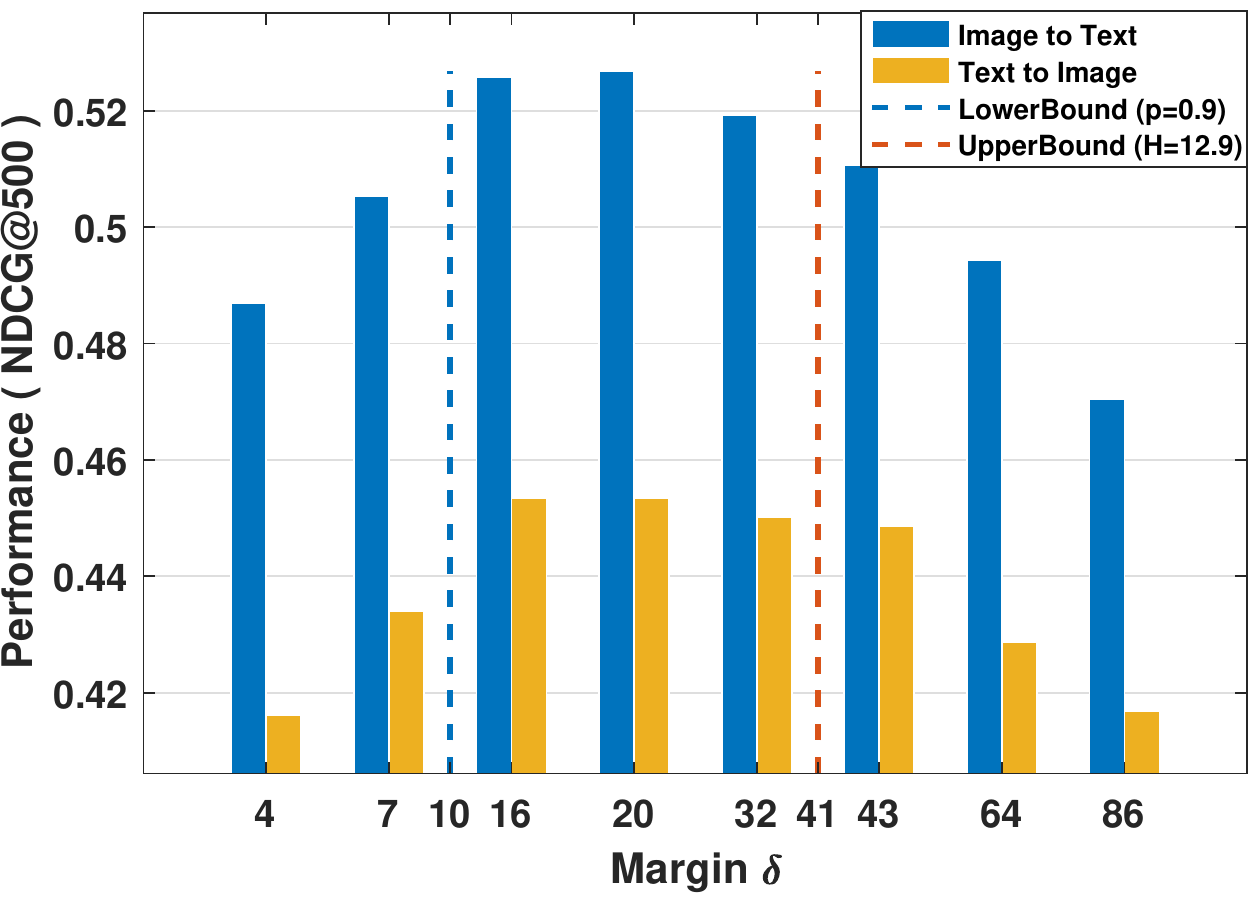}&
\includegraphics[scale=.45]{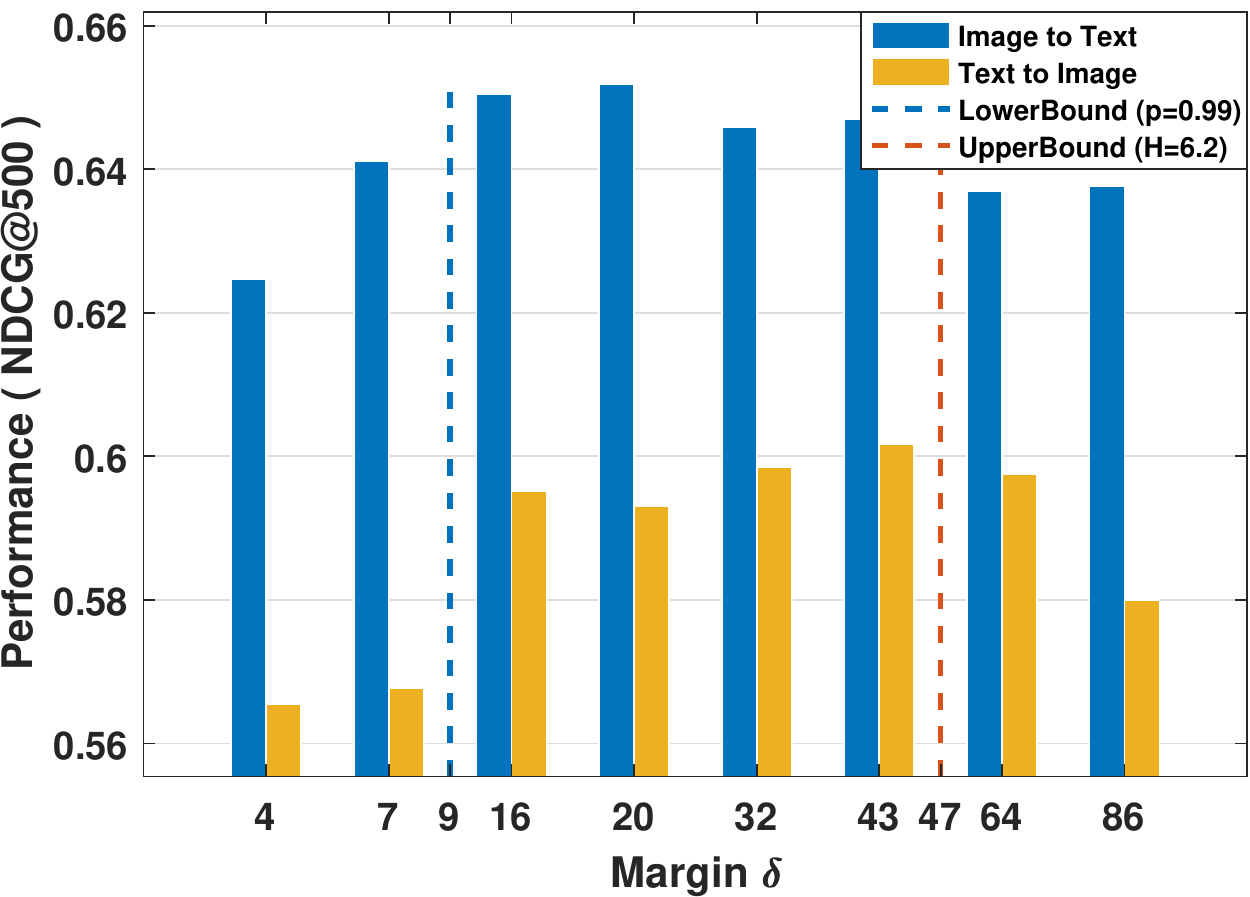}&
\includegraphics[scale=.45]{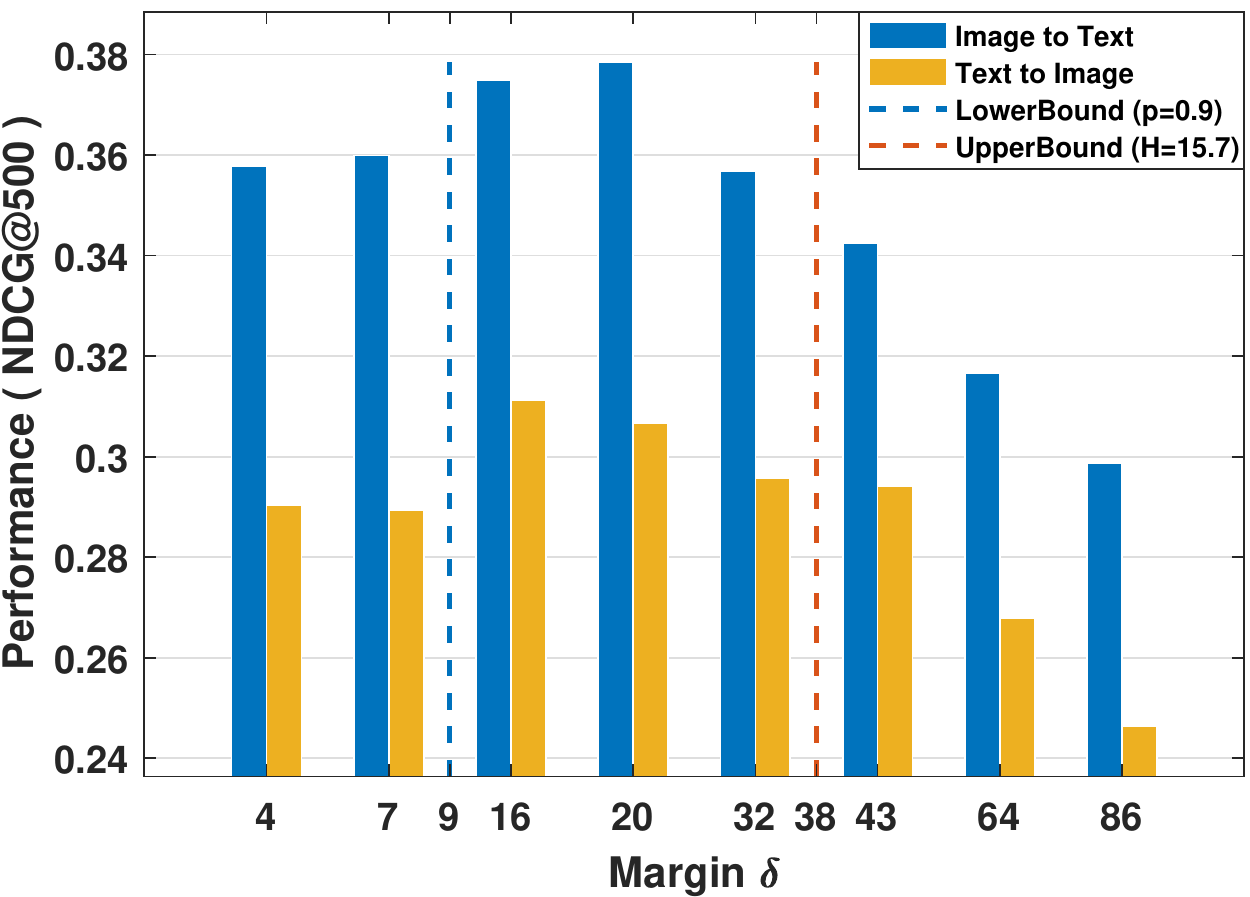}\\
@MIFLICKR25K&@NUS-WIDE& @MS COCO\\
\end{tabular}
\caption{The impact of different robust parameter $\delta$ in the
RMSH method on MIRFLICKR25K, NUS-WIDE, and MS COCO datasets. The
lower and upper bounds of effective $\delta$ are marked. The code
length is 128.} \label{fig:effective_margin_bound}
\end{figure*}

\subsection{Experimental Results}
\label{sec:exp3} Table \ref{tab:tab_compare_cross_hash} reports the
NDCG@500 results of compared methods and the precision-recall cures
with code length 64 on three datasets are shown in
Fig.~\ref{fig:campare_of_PR}. From Table
\ref{tab:tab_compare_cross_hash}, we observe that our RMSH method
substantially outperforms other compared methods on all used
datasets. Specifically, compared to the best deep method SVHN, our
RMSH obtains the relative increase of 0.2\%$\sim$5.86\%,
0.1\%$\sim$2.58\%, and 0.1\%$\sim$7.94\% with different bits for two
cross-modal retrieval tasks on MIRFLICKR25K, NUS-WIDE, and MS COCO
datasets, respectively. Because the SVHN method is learning to
preserve the binary cross-modal similarity structures, its hashing
codes cannot capture the fine-grained ranking information, thus
achieve inferior NDCG scores. In contrast, the RMSH learns the more
complicated similarity with more robustness to the modality
discrepancy and noise. From Fig.~\ref{fig:campare_of_PR}, we see
that the precision-recall evaluations are consistent with the NDCG
scores for two cross-modal retrieval tasks. These results re-confirm
the superiority of our RMSH method.

Fig.~\ref{fig:retrieval_example_coco_TI} and
Fig.~\ref{fig:retrieval_example_coco_IT} show the 'Text to Image'
and 'Image to Text' search example of compared methods on COCO
dataset. As can be seen, the RMSH method tends to retrieve more
relevant images than others for the query containing certain
concepts, e.g., person and dog. From
Fig.~\ref{fig:retrieval_example_coco_IT}, we observe that the RMSH
method can roughly capture the similarity across images and text
descriptions with multiple tags. Compared to other methods, the top
5 retrieved textual results of RMSH are more meaningful and
relevant.

To visualize the quality of learned codes by RMSH, we use t-SNE
tools to embed the 128 bits testing binary-like features of NUS-WIDE
datasets into 2-dimension spaces and visualize their distribution in
Fig.~\ref{fig:visualize_distribution}. As can be seen, both image
and text features provide a better separation between different
categories in the common space. Also, the features belonging to the
same class from different modalities appear to be compact. These
results indicate that the RMSH can preserve not only the multilevel
semantic similarity of intra-modal data but also that of the
inter-modal data.

\subsection{Analysis and Discussion}
\subsubsection{Impact of robust parameter $\delta$.}
To validate the correctness of the derived bounds of effective
robust parameter $\delta$ in Sec.~\ref{sec:emm}, we separately tune
$\delta$ in $\{K/40$, $K/20$, $K/8$, $3K/20$, $K/4$, $K/3$, $K/2$,
$2K/3$ $\}$ with other parameters fixed and report the retrieval
performance in Fig.~\ref{fig:effective_margin_bound}, where K=128.
We also mark the bounds of effective $\delta$ derived by
Eq.~(\ref{eq:upper_bound}) and (\ref{eq:lower_bound}) in
Fig.~\ref{fig:effective_margin_bound}. We see that the NDCG scores
of the RMSH method on all datasets are higher when setting $\delta$
in the range of bounds than when setting it out of the range. This
result experimentally proves the correctness of the bounds for more
robust cross-modal hash learning.

Besides, Fig.~\ref{fig:distance_of_margins} shows the distribution
of the Hamming distance of learned codes by RMSH. We observe that 1)
too small $\delta$ makes the dissimilar points be nearer, which
causes the difficulties of keeping multilevel similarity structure
of the similar points; 2) too large $\delta$ prefers to scatter
dissimilar points in the whole Hamming space, which is cline to
cause false coding since the number of different codes that satisfy
the distance constraint is limited.

\begin{table*}
\caption{Ablation study of RMSH in terms of NDCG@500 scores on
MIRFLICKR25K, NUS-WIDE, and MS COCO datasets. The best result is
shown in boldface.} \label{tab:influence_of_components}
\resizebox{\textwidth}{!}{%
\begin{tabular}{|c|c|c|c|c|c|c|c|c|c|c|c|c|c|}
\hline \multirow{2}{*}{Task} & \multirow{2}{*}{Method} &
\multicolumn{4}{c|}{MIRFLICKR25K} & \multicolumn{4}{c|}{NUS-WIDE} &
\multicolumn{4}{c|}{MS COCO} \\ \cline{3-14}
 &  & 16bits & 32bits & 64bits & 128bits & 16bits & 32bits & 64bits & 128bits & 16bits & 32bits & 64bits & 128bits \\ \hline
\multirow{4}{*}{I $\to$ T} & RMSH-T & 0.4053 & 0.4202 & 0.4524 & 0.4378 & 0.5437 & 0.5894 & 0.5894 & 0.5885 & 0.2759 & 0.3147 & 0.3430 & 0.3303 \\
 & RMSH-C & 0.4066 & 0.4143 & 0.4331 & 0.4597 & 0.5946 & 0.6071 & 0.6383 & \textbf{0.6593} & 0.1592 & 0.1811 & 0.2038 & 0.2218 \\
 & RMSH-P & 0.4290 & 0.4598 & 0.4920 & 0.5082 & 0.6042 & 0.6235 & 0.6248 & 0.6492 & 0.2918 & 0.3253 & 0.3331 & 0.3155 \\
 & RMSH &\textbf{0.4592} & \textbf{0.4940} & \textbf{0.5180} & \textbf{0.5183} & \textbf{0.6148} & \textbf{0.6401} & \textbf{0.6497} & 0.6574 & \textbf{0.3037} & \textbf{0.3724} & \textbf{0.3855} & \textbf{0.3998}\\ \hline
\multirow{4}{*}{T $\to$ I} & RMSH-T & 0.3905 & 0.4169 & 0.4061 & 0.3677 & 0.5369 & 0.5560 & 0.5467 & 0.5583 & 0.2575 & 0.2725 & 0.2978 & 0.3079 \\
 & RMSH-C & 0.3428 & 0.3696 & 0.3809 & 0.4071 & 0.5467 & 0.5502 & 0.5803 & \textbf{0.6069} & 0.1625 & 0.1714 & 0.1868 & 0.1983 \\
 & RMSH-P & 0.3912 & 0.4177 & 0.4226 & 0.4453 & 0.5580 & 0.5725 & 0.5723 & 0.5823 & 0.2587 & 0.2747 & 0.2774 & 0.2742 \\
& RMSH & \textbf{0.4379} & \textbf{0.4533} & \textbf{0.4562} &
\textbf{0.4628} & \textbf{0.5744} & \textbf{0.5921} &
\textbf{0.5900} & 0.5999 & \textbf{0.2897} & \textbf{0.3080} &
\textbf{0.3050} & \textbf{0.3175} \\ \hline
\end{tabular}%
}
\end{table*}
\subsubsection{Ablation Study.}
To analyze the effectiveness of different loss terms and the
pseudo-codes network in the proposed RMSH method, we separately
remove: the classification loss term $\mathcal{L}_{cl}$, the triplet
loss term $\mathcal{L}_{tr}$ and the pseudo-codes with others
remained to evaluate their influence on the final performance. These
three models are called {\bf RMSH-C}, {\bf RMSH-T}, and {\bf
RMSH-P}. Table \ref{tab:influence_of_components} shows the result.
We see that separately removing them will damage the retrieval
performance to varying degrees. Notably, 1) the performance gap
between RMSH and RMSH-T enlarges with the code length increasing on
the MIRFLICKR25K dataset, which confirms the importance of the
margin-adaptive triplet loss for preserving the fine-grained
similarities of multi-label data in larger Hamming space. 2) As the
code length increases on the MIRFLICKR25K and NUS-WIDE datasets, the
triplet loss term shows more influence on the final performance than
the classification loss term, which indicates that the exploitation
of ranking information is more critical than the discriminative
information to capture similarity structure.  3) Because the
similarity information on the MS COCO dataset (80 concepts) is more
sparse, the improvement introduced by the strategy of the
pseudo-codes is more significant, i.e., the difference between RMSH
and RMSH-P.

\begin{figure}
\center
\begin{tabular}{cc}
\includegraphics[scale=.32]{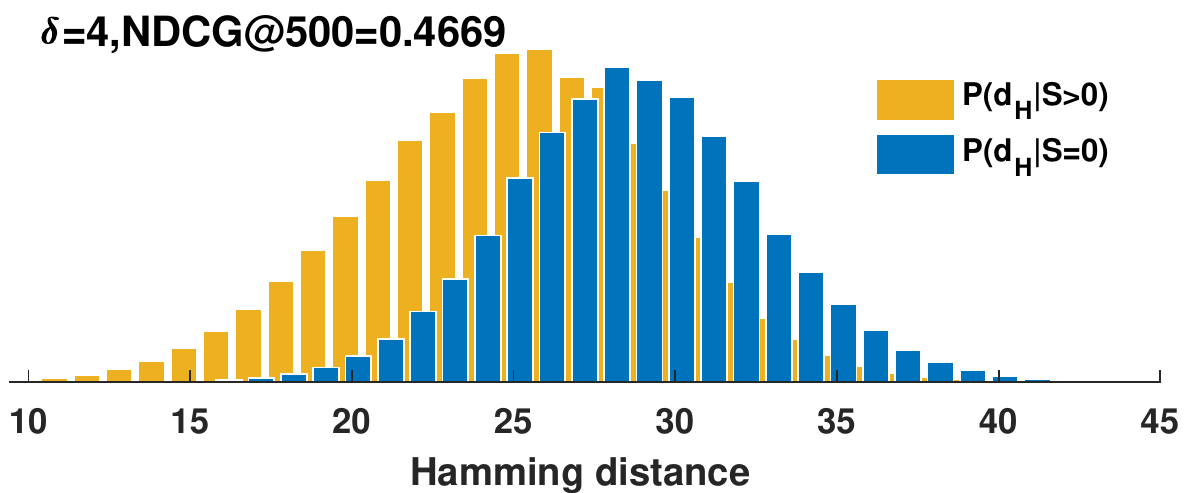}&
\includegraphics[scale=.32]{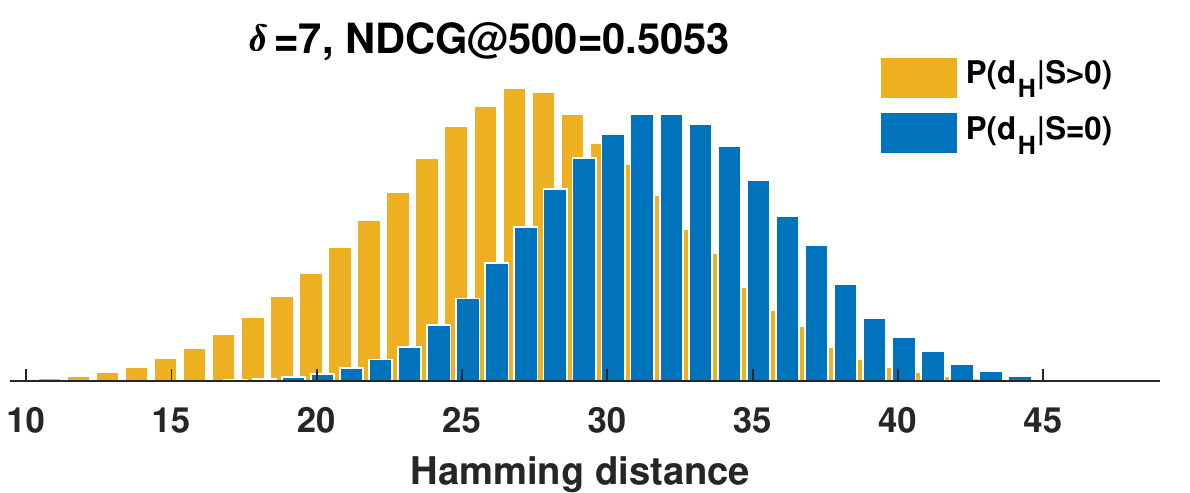}\\
\includegraphics[scale=.32]{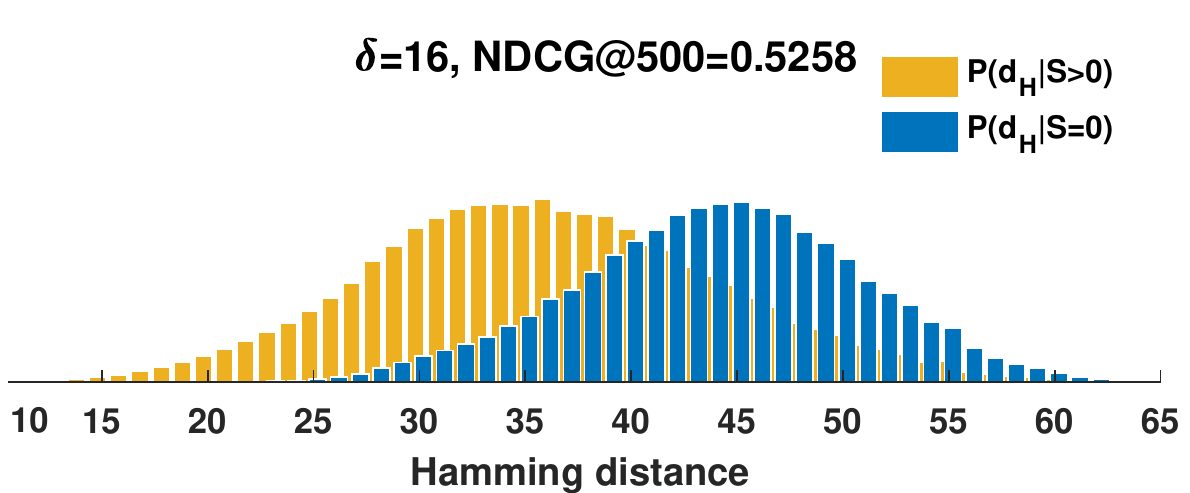}&
\includegraphics[scale=.32]{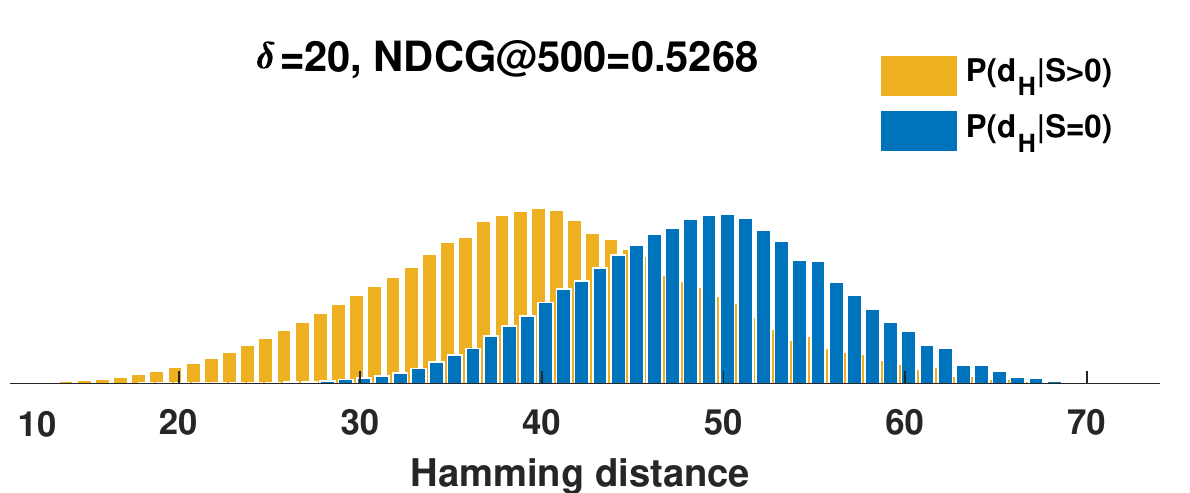}\\
\includegraphics[scale=.32]{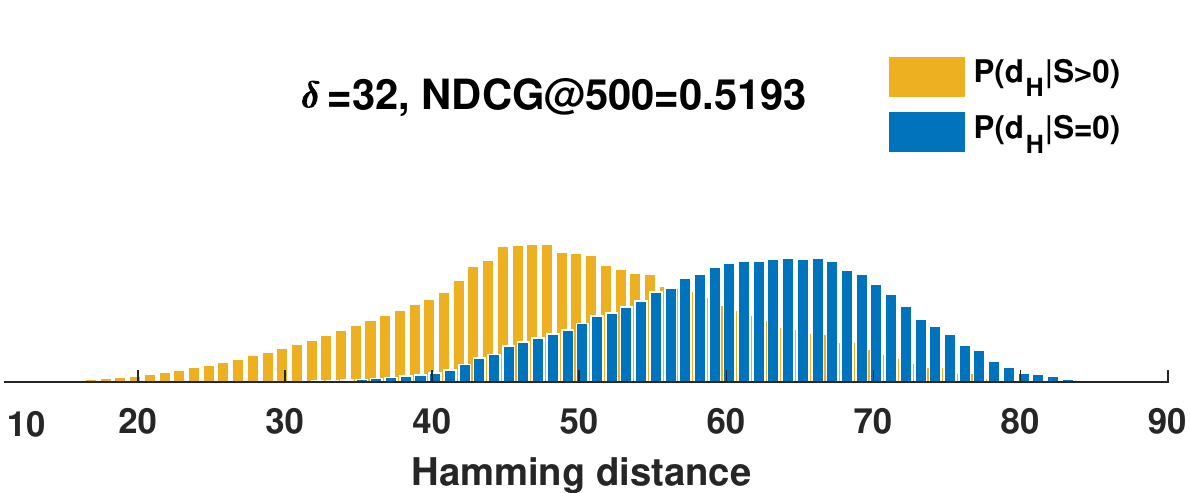}&
\includegraphics[scale=.32]{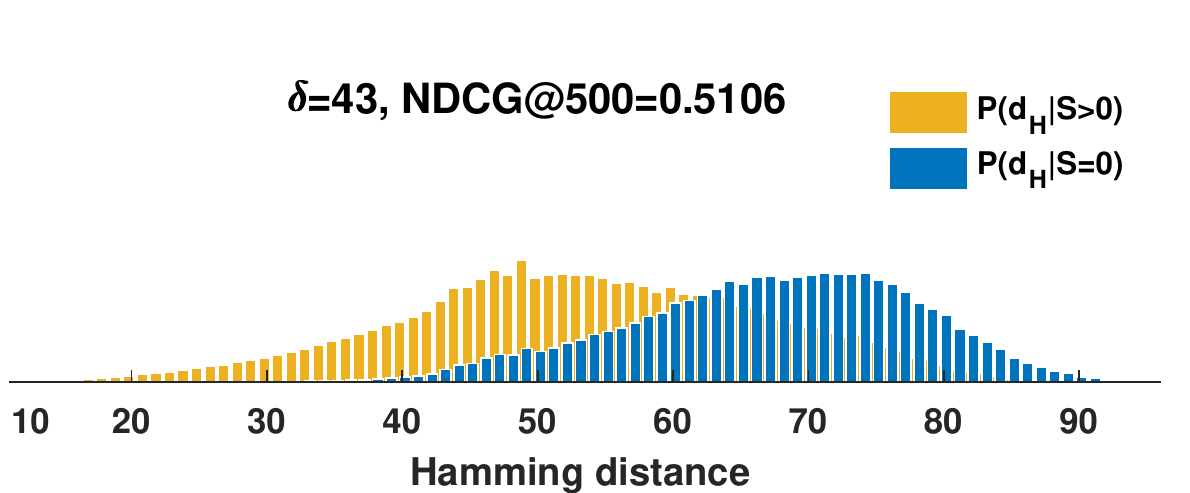}\\
\includegraphics[scale=.32]{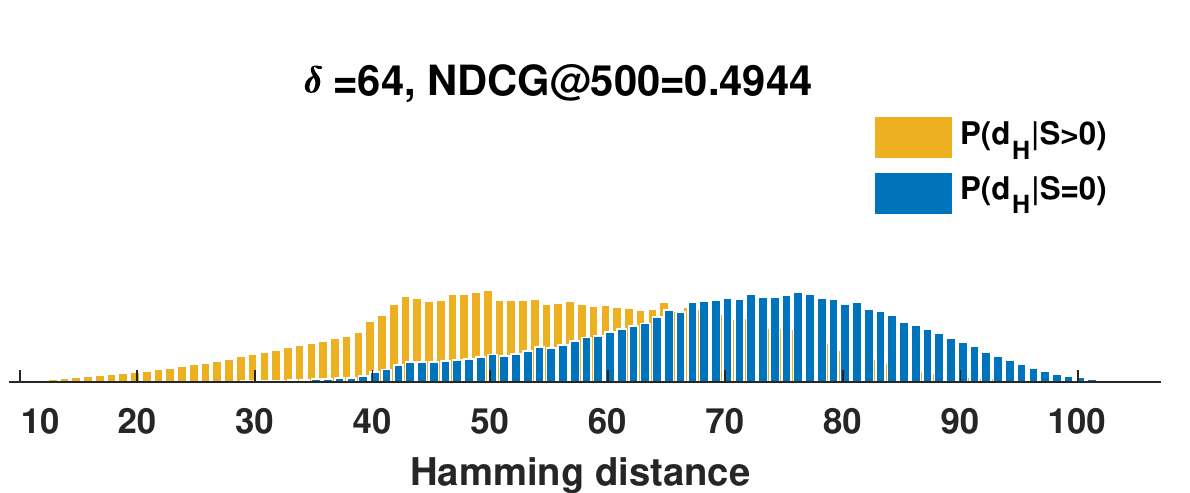}&
\includegraphics[scale=.32]{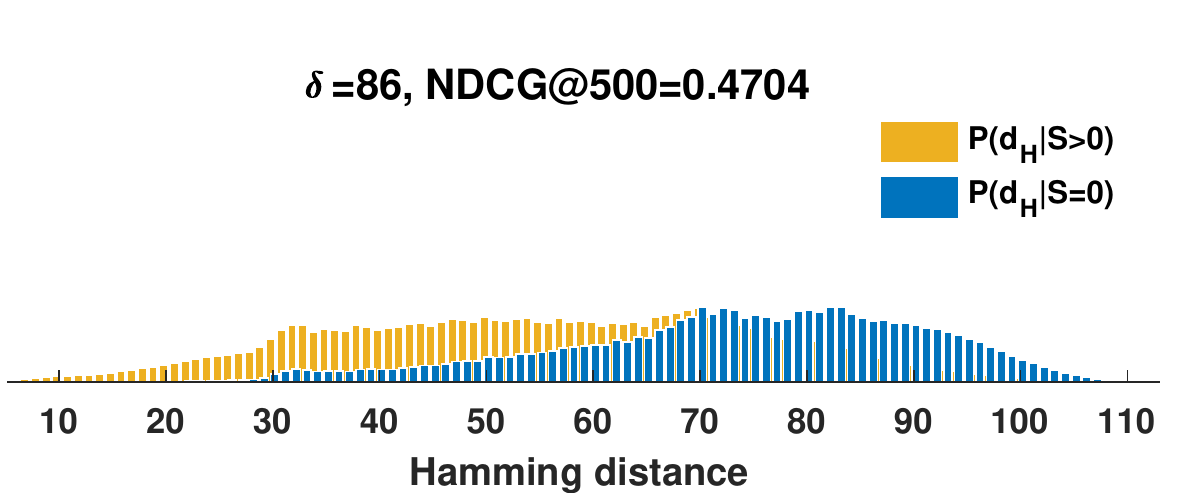}\\
\end{tabular}
\caption{The conditional distributions (averaged on the MIRFLICKR25K
image test set ) of Hamming distance between testing image codes and
training text codes when given semantic similarity, after optimizing
RMSH with different robust parameter $\delta$.}
\label{fig:distance_of_margins}
\end{figure}

\begin{figure}
\center
\begin{tabular}{cc}
\includegraphics[scale=.3]{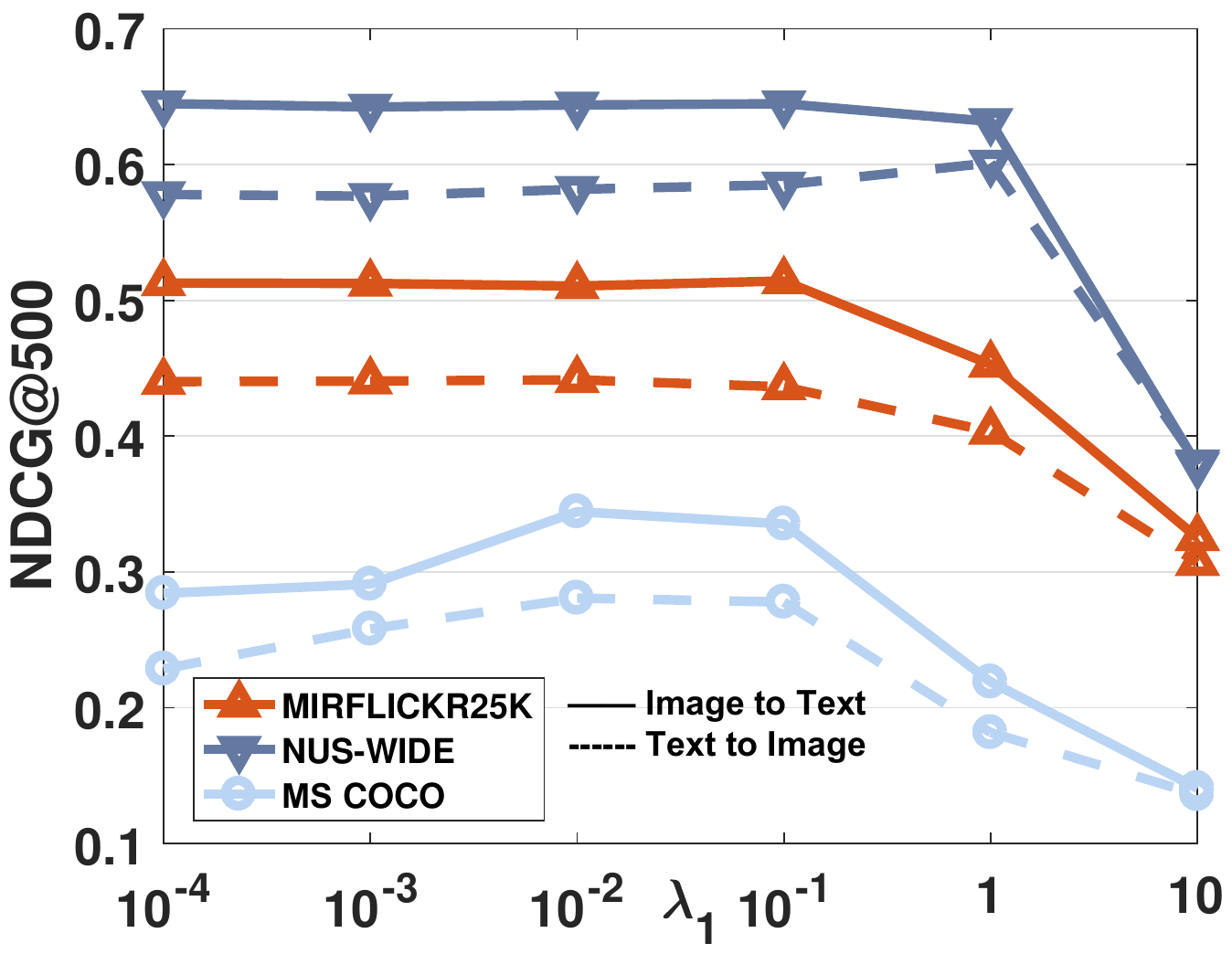}&
\includegraphics[scale=.3]{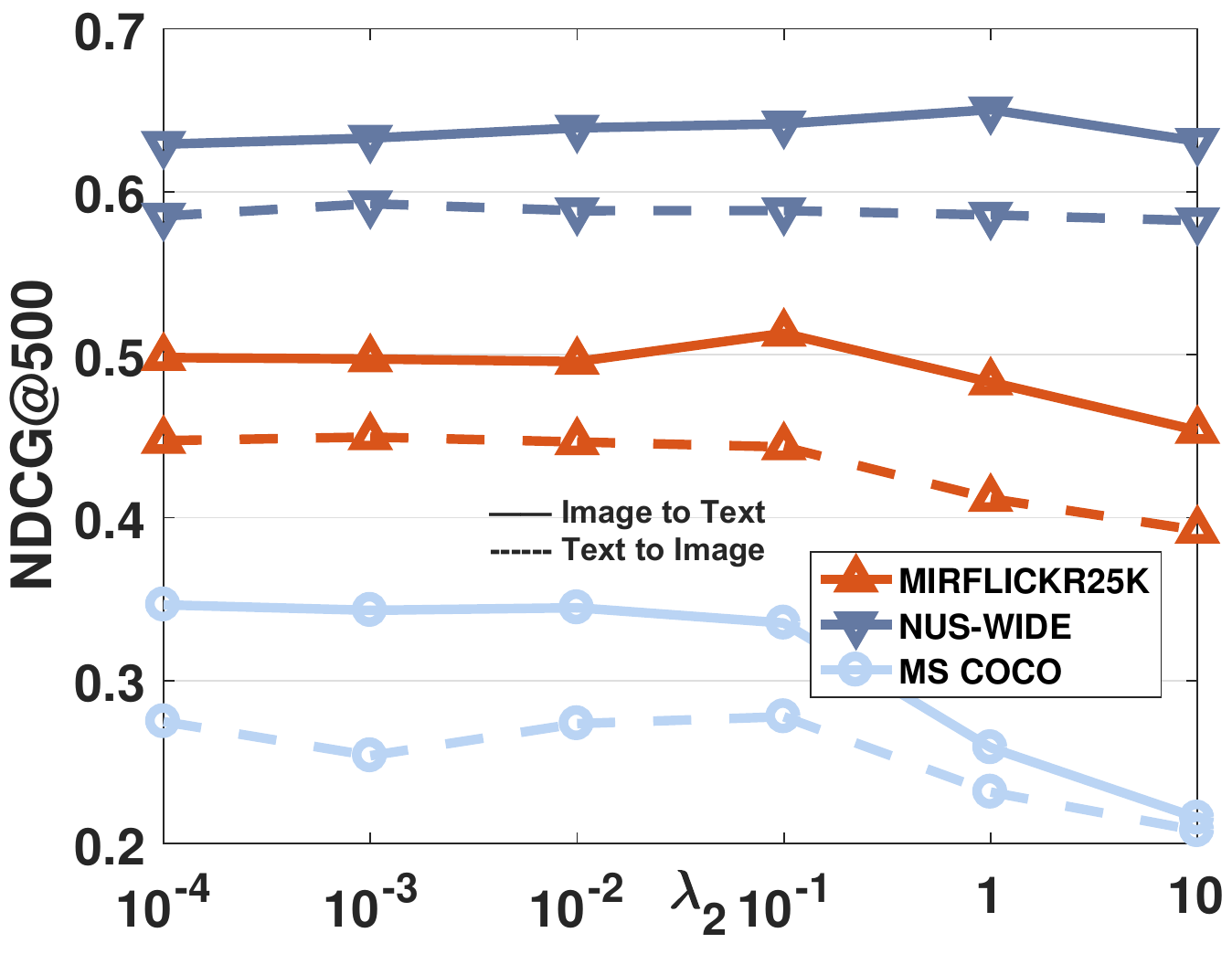}\\
\includegraphics[scale=.3]{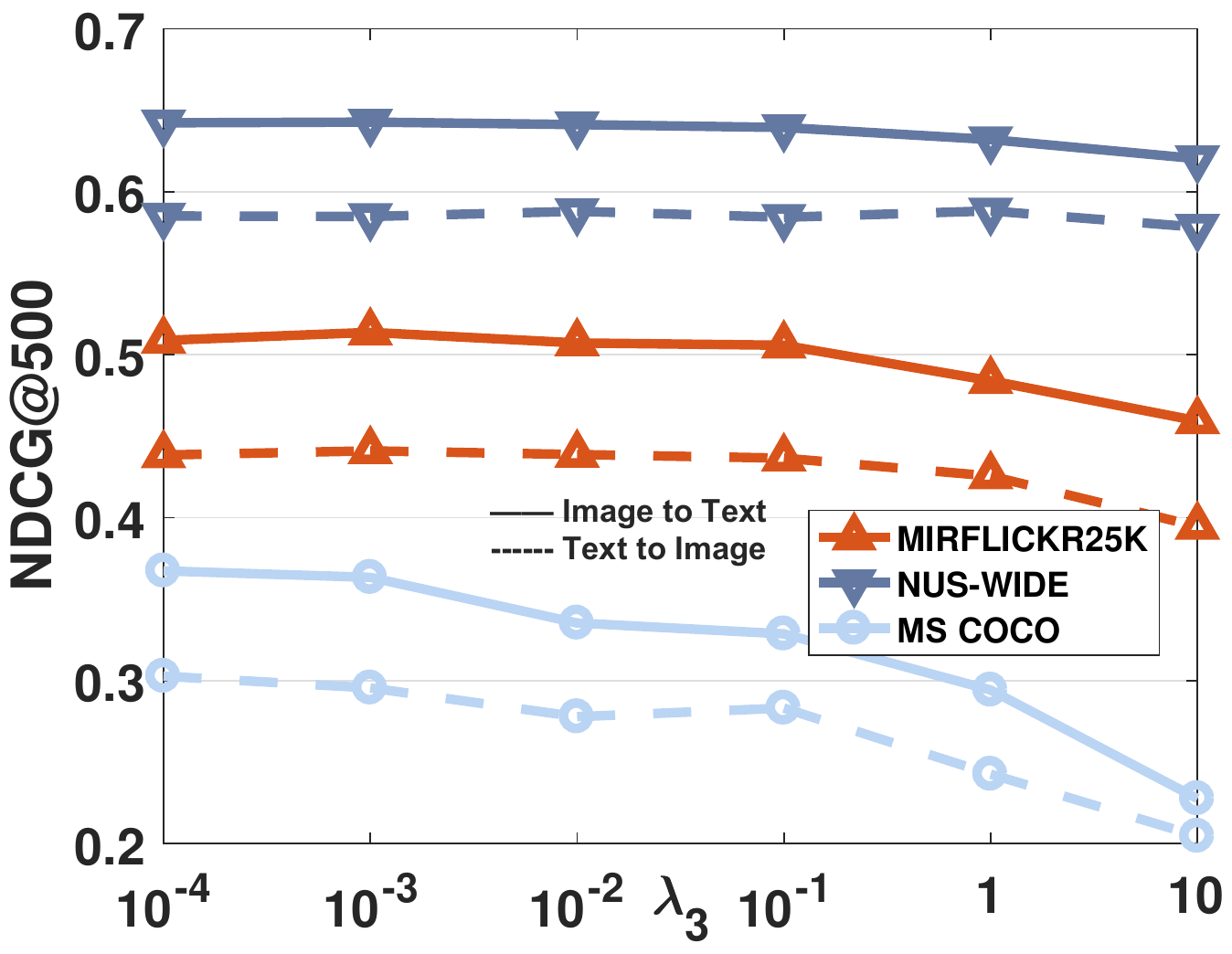}&
\includegraphics[scale=.3]{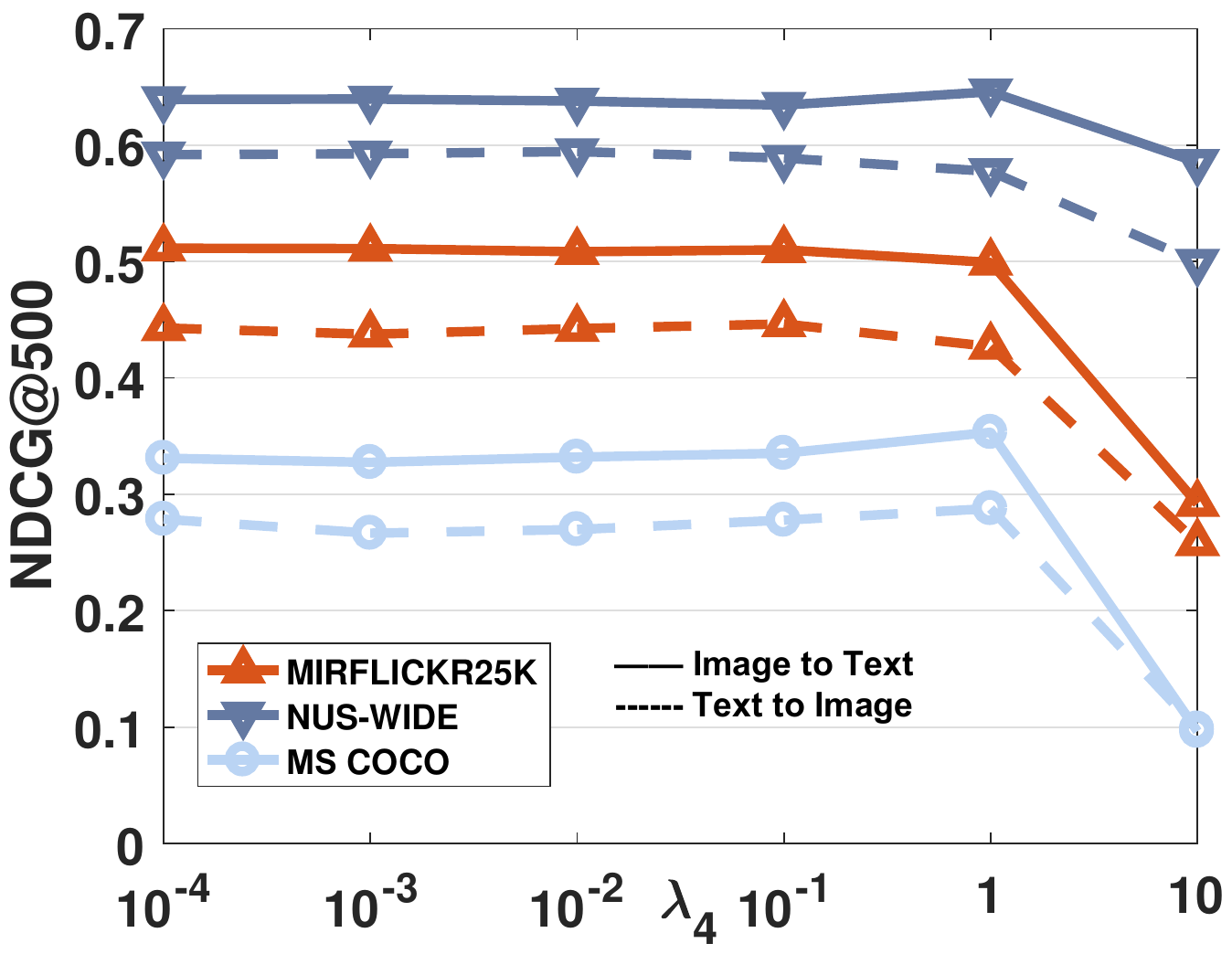}\\
\end{tabular}
\caption{Parameter analysis of $\lambda_1$, $\lambda_2$, $\lambda_3$
and $\lambda_4$ on MIRFLICKR25K, NUS-WIDE, and COCO dataset.}
\label{fig:influence_of_parameters}
\end{figure}

\subsection{Parameter Analysis.}
In this section, we conducted the experiments to analyze the
parameter sensitive of RMSH on three datasets by varying the one
parameter while fixing others. There are mainly four
hyper-parameters in RMSH, i.e., $\lambda_1$, $\lambda_2$,
$\lambda_3$, and $\lambda_4$. The NDCG@500 values with different
settings of hyper-parameters in the case of 64 bits on two
cross-modal retrieval tasks are summarized in
Figure~\ref{fig:influence_of_parameters}. From the figure, 1) we
observe that the NDCG scores first obtain a slight increase and then
degrades as $\lambda_1$ and $\lambda_2$ becomes larger,
respectively. This indicates that proper $\lambda_1$ and $\lambda_2$
setting is helpful to preserve intra-modal and inter-modal
similarities simultaneously, a huge gap between the value of
$\lambda_1$ and $\lambda_2$ will break this balance and hurt the
final performance. 2)The performance emerges decrease when
$\lambda_3$ gets large since weighing more the classification loss
of pseudo-codes will introduce more noise to the optimization
process. 3) The result of $\lambda_4$ is similar to that of
$\lambda_1$. A suitable $\lambda_4$ is useful to reduce the
quantization loss, but a large $\lambda_4$ will weaken the goal of
preserving similarity.

\section{Conclusion}
We have presented a novel robust multilevel semantic hashing for
multi-label cross-modal retrieval. The approach preserves the
multilevel semantic similarity of data and explicitly ensures the
distance between different codes is larger than a specific value for
robustness to the modality discrepancy. Mainly, we give an effective
range of this value from the information coding theory analysis and
characterize the above goal as a margin-adaptive triplet loss. We
further introduce a pseudo-codes network for the imbalanced
semantics. Our approach yields the state-of-the-art empirical
results on three benchmarks.


\section*{Acknowledgements}
This work is partially supported by National Science Foundation of
China (61976115,61672280, 61732006), AI+ Project of
NUAA(NZ2020012,56XZA18009), China Scholarship Council(201906830057).

\bibliographystyle{elsarticle-num}
\bibliography{ref_pr}

\begin{thebibliography}{10}
\expandafter\ifx\csname url\endcsname\relax
  \def\url#1{\texttt{#1}}\fi
\expandafter\ifx\csname urlprefix\endcsname\relax\def\urlprefix{URL }\fi
\expandafter\ifx\csname href\endcsname\relax
  \def\href#1#2{#2} \def\path#1{#1}\fi

\bibitem{cao2019video}
D.~Cao, Z.~Yu, H.~Zhang, J.~Fang, L.~Nie, Q.~Tian, Video-based cross-modal
  recipe retrieval, in: Proc. ACM MM, 2019, pp. 1685--1693.

\bibitem{wang2019camp}
Z.~Wang, X.~Liu, H.~Li, L.~Sheng, J.~Yan, X.~Wang, J.~Shao, Camp: Cross-modal
  adaptive message passing for text-image retrieval, in: Proc. ICCV, 2019, pp.
  5764--5773.

\bibitem{dutta2019semantically}
A.~Dutta, Z.~Akata, Semantically tied paired cycle consistency for zero-shot
  sketch-based image retrieval, in: Proc. CVPR, 2019, pp. 5089--5098.

\bibitem{zhu2019r2gan}
B.~Zhu, C.-W. Ngo, J.~Chen, Y.~Hao, R2gan: Cross-modal recipe retrieval with
  generative adversarial network, in: Proc. CVPR, 2019, pp. 11477--11486.

\bibitem{BaltrusaitisAM19}
T.~Baltrusaitis, C.~Ahuja, L.~Morency, Multimodal machine learning: {A} survey
  and taxonomy, {IEEE} Trans. Pattern Anal. Mach. Intell. 41~(2) (2019)
  423--443.

\bibitem{XuLYDL19}
R.~Xu, C.~Li, J.~Yan, C.~Deng, X.~Liu, Graph convolutional network hashing for
  cross-modal retrieval, in: Proc. IJCAI, 2019, pp. 982--988.

\bibitem{ZhenHWP19}
L.~Zhen, P.~Hu, X.~Wang, D.~Peng, Deep supervised cross-modal retrieval, in:
  Proc. CVPR, 2019, pp. 10394--10403.

\bibitem{Zhan0YZT018}
Y.~Zhan, J.~Yu, Z.~Yu, R.~Zhang, D.~Tao, Q.~Tian, Comprehensive
  distance-preserving autoencoders for cross-modal retrieval, in: Proc. ACM MM,
  2018, pp. 1137--1145.

\bibitem{LiuLZWHJ18}
H.~Liu, M.~Lin, S.~Zhang, Y.~Wu, F.~Huang, R.~Ji, Dense auto-encoder hashing
  for robust cross-modality retrieval, in: Proc. ACM MM, 2018, pp. 1589--1597.

\bibitem{ShiYZWP19}
Y.~Shi, X.~You, F.~Zheng, S.~Wang, Q.~Peng, Equally-guided discriminative
  hashing for cross-modal retrieval, in: Proc. IJCAI, 2019, pp. 4767--4773.

\bibitem{Jiang2017Deep}
Q.~Jiang, W.~Li, Deep cross-modal hashing, in: Proc. CVPR, 2017, pp.
  3270--3278.

\bibitem{WangZSSS18}
J.~Wang, T.~Zhang, J.~Song, N.~Sebe, H.~T. Shen, A survey on learning to hash,
  {IEEE} Trans. Pattern Anal. Mach. Intell. 40~(4) (2018) 769--790.

\bibitem{HuWZP19}
P.~Hu, X.~Wang, L.~Zhen, D.~Peng, Separated variational hashing networks for
  cross-modal retrieval, in: Proc. ACM MM, 2019, pp. 1721--1729.

\bibitem{cao2018cross}
Y.~Cao, B.~Liu, M.~Long, J.~Wang, Cross-modal hamming hashing, in: Proc. ECCV,
  2018, pp. 202--218.

\bibitem{zhang2018attention}
X.~Zhang, H.~Lai, J.~Feng, Attention-aware deep adversarial hashing for
  cross-modal retrieval, in: Proc. ECCV, 2018, pp. 591--606.

\bibitem{chen2019two}
Z.-D. Chen, Y.~Wang, H.-Q. Li, X.~Luo, L.~Nie, X.-S. Xu, A two-step cross-modal
  hashing by exploiting label correlations and preserving similarity in both
  steps, in: Proc. ACM MM, 2019, pp. 1694--1702.

\bibitem{lu2019flexible}
X.~Lu, L.~Zhu, Z.~Cheng, J.~Li, X.~Nie, H.~Zhang, Flexible online multi-modal
  hashing for large-scale multimedia retrieval, in: Proc. ACM MM, 2019, pp.
  1129--1137.

\bibitem{li2019coupled}
C.~Li, C.~Deng, L.~Wang, D.~Xie, X.~Liu, Coupled cyclegan: Unsupervised hashing
  network for cross-modal retrieval, in: Proc. AAAI, Vol.~33, 2019, pp.
  176--183.

\bibitem{su2019deep}
S.~Su, Z.~Zhong, C.~Zhang, Deep joint-semantics reconstructing hashing for
  large-scale unsupervised cross-modal retrieval, in: Proc. ICCV, 2019, pp.
  3027--3035.

\bibitem{Ding2014Collective}
G.~Ding, Y.~Guo, J.~Zhou, Collective matrix factorization hashing for
  multimodal data, in: Proc. CVPR, 2014, pp. 2083--2090.

\bibitem{Zhang2014Large}
D.~Zhang, W.~Li, Large-scale supervised multimodal hashing with semantic
  correlation maximization, in: Proc. AAAI, 2014, pp. 2177--2183.

\bibitem{Lin2015Semantics}
Z.~Lin, G.~Ding, M.~Hu, J.~Wang, Semantics-preserving hashing for cross-view
  retrieval, in: Proc. CVPR, 2015, pp. 3864--3872.

\bibitem{LiongLT017}
V.~E. Liong, J.~Lu, Y.~Tan, J.~Zhou, Cross-modal deep variational hashing, in:
  Proc. ICCV, 2017, pp. 4097--4105.

\bibitem{LiDL0GT18}
C.~Li, C.~Deng, N.~Li, W.~Liu, X.~Gao, D.~Tao, Self-supervised adversarial
  hashing networks for cross-modal retrieval, in: Proc. CVPR, 2018, pp.
  4242--4251.

\bibitem{guruswami2012essential}
V.~Guruswami, A.~Rudra, M.~Sudan, Essential coding theory, Draft available at
  http://www. cse. buffalo. edu/~ atri/courses/coding-theory/book.

\bibitem{HelgertS73}
H.~J. Helgert, R.~D. Stinaff, Minimum-distance bounds for binary linear codes,
  {IEEE} Trans. Information Theory 19~(3) (1973) 344--356.

\bibitem{AgrellVZ01}
E.~Agrell, A.~Vardy, K.~Zeger, A table of upper bounds for binary codes, {IEEE}
  Trans. Information Theory 47~(7) (2001) 3004--3006.

\bibitem{BelliniGS14}
E.~Bellini, E.~Guerrini, M.~Sala, Some bounds on the size of codes, {IEEE}
  Trans. Information Theory 60~(3) (2014) 1475--1480.

\bibitem{gilbert1952comparison}
E.~N. Gilbert, A comparison of signalling alphabets, The Bell system technical
  journal 31~(3) (1952) 504--522.

\bibitem{AlfassyKASHFGB19}
A.~Alfassy, L.~Karlinsky, A.~Aides, J.~Shtok, S.~Harary, R.~S. Feris,
  R.~Giryes, A.~M. Bronstein, Laso: Label-set operations networks for
  multi-label few-shot learning, in: Proc. CVPR, 2019, pp. 6548--6557.

\bibitem{Wu2017Deep}
D.~Wu, Z.~Lin, B.~Li, M.~Ye, W.~Wang, Deep supervised hashing for multi-label
  and large-scale image retrieval, in: Proc. ICMR, 2017, pp. 150--158.

\bibitem{Lin2014Microsoft}
T.~Lin, M.~Maire, S.~J. Belongie, J.~Hays, P.~Perona, D.~Ramanan,
  P.~Doll{\'{a}}r, C.~L. Zitnick, Microsoft coco: Common objects in context,
  in: Proc. ECCV, 2014, pp. 740--755.

\bibitem{HuiskesL08mir}
M.~J. Huiskes, M.~S. Lew, The {MIR} flickr retrieval evaluation, in:
  Proceedings of the 1st {ACM} {SIGMM} International Conference on Multimedia
  Information Retrieval, {MIR} 2008, Vancouver, British Columbia, Canada,
  October 30-31, 2008, 2008, pp. 39--43.

\bibitem{civr09nuswide}
T.-S. Chua, J.~Tang, R.~Hong, H.~Li, Z.~Luo, Y.-T. Zheng, Nus-wide: A
  real-world web image database from national university of singapore, in:
  Proc. of ACM Conf. on Image and Video Retrieval (CIVR'09), July 8-10, 2009.

\bibitem{Kalervo2000IR}
K.~J{\"{a}}rvelin, J.~Kek{\"{a}}l{\"{a}}inen, Ir evaluation methods for
  retrieving highly relevant documents, in: International ACM SIGIR Conference
  on Research and Development in Information Retrieval, 2000, pp. 41--48.

\bibitem{zhao2015deep}
F.~Zhao, Y.~Huang, L.~Wang, T.~Tan, Deep semantic ranking based hashing for
  multi-label image retrieval, in: Proc. CVPR, 2015, pp. 1556--1564.

\bibitem{DengCLGT18}
C.~Deng, Z.~Chen, X.~Liu, X.~Gao, D.~Tao, Triplet-based deep hashing network
  for cross-modal retrieval, {IEEE} Trans. Image Processing 27~(8) (2018)
  3893--3903.

\bibitem{He2015Deep}
K.~He, X.~Zhang, S.~Ren, J.~Sun, Deep residual learning for image recognition,
  in: Proc. CVPR, 2016, pp. 770--778.

\bibitem{KirosZSZUTF15}
R.~Kiros, Y.~Zhu, R.~Salakhutdinov, R.~S. Zemel, R.~Urtasun, A.~Torralba,
  S.~Fidler, Skip-thought vectors, in: Proc. NeuraIPS, 2015, pp. 3294--3302.

\end{thebibliography}

\end{document}